\documentclass{article}

\PassOptionsToPackage{numbers}{natbib}
\usepackage[preprint]{neurips_2026}

\usepackage[utf8]{inputenc} 
\usepackage[T1]{fontenc}    
\usepackage{hyperref}       
\usepackage{url}            
\usepackage{booktabs}       
\usepackage{amsfonts}       
\usepackage{nicefrac}       
\usepackage{microtype}      
\usepackage{xcolor}         
\usepackage[ruled,vlined]{algorithm2e}

\usepackage[textsize=tiny]{todonotes}
\usepackage{xspace}
\usepackage{comment}
\usepackage{multirow}
\usepackage{array}
\usepackage{makecell}
\usepackage{enumitem}
\usepackage{tcolorbox}
\usepackage{amsmath}
\usepackage{amssymb}
\usepackage{mathtools}
\usepackage{amsthm}
\usepackage{wrapfig}
\usepackage{subcaption}
\usepackage[table]{xcolor}
\usepackage{multirow}

\theoremstyle{plain}
\newtheorem{theorem}{Theorem}[section]

\newtheorem{corollary}[theorem]{Corollary}
\theoremstyle{definition}

\theoremstyle{remark}

\newcommand{\frameworkname}{xMemory\xspace}
\usepackage{xcolor}

\usepackage[normalem]{ulem}
\usepackage{xcolor}

\title{Beyond RAG for Agent Memory: Retrieval by Decoupling and Aggregation}

\author{%
  Zhanghao Hu$^{1,2}$\thanks{Equal contribution.} \quad
  Qinglin Zhu$^{1}$\footnotemark[1] \quad
  Runcong Zhao$^{1}$ \quad
  Di Liang$^{2}$ \\
  Hanqi Yan$^{1}$ \quad
  Yulan He$^{1}$ \quad
  Lin Gui$^{1}$ \\
  \\
  $^{1}$King's College London \\
  $^{2}$Tencent, Yuanbao Team \\
  \texttt{\{zhanghao.hu,  lin.1.gui\}@kcl.ac.uk}
}

\begin{document}

\maketitle

\begin{abstract}
Standard Retrieval Augmented Generation (RAG) is poorly matched to agent memory. Unlike large heterogeneous corpora, agent memory forms a bounded and coherent interaction stream in which many spans are highly correlated or near duplicates. As a result, flat top-$k$ similarity retrieval often returns redundant context, while summary-centric hierarchies can blur the subtle details that distinguish one candidate from another. We argue that agent memory should follow the principle of \emph{decoupling before aggregation}: the system should first isolate reusable facts, updates, and distinguishing details from similar histories, and only then organise them for efficient retrieval. Based on this principle, we propose \frameworkname, which constructs a revisable hierarchical memory structure from original messages to segments, memory components, and groups. \frameworkname segments interaction history into local events, decouples each segment into memory components, aggregates related components into high-level groups using a sparsity--semantic faithfulness objective, and maintains this structure incrementally as memory evolves. At inference time, \frameworkname retrieves top-down, first selecting a compact backbone of complementary groups and components, and then expanding to segments and raw messages only when additional evidence reduces the reader's uncertainty. Experiments on LoCoMo and PerLTQA across diverse open source and closed source LLMs show consistent gains in answer quality and inference token efficiency, supported by analyses of redundancy, evidence density, and coverage.
\end{abstract}

\section{Introduction}

\begin{figure*}[t]
\centering
\includegraphics[width=\linewidth]{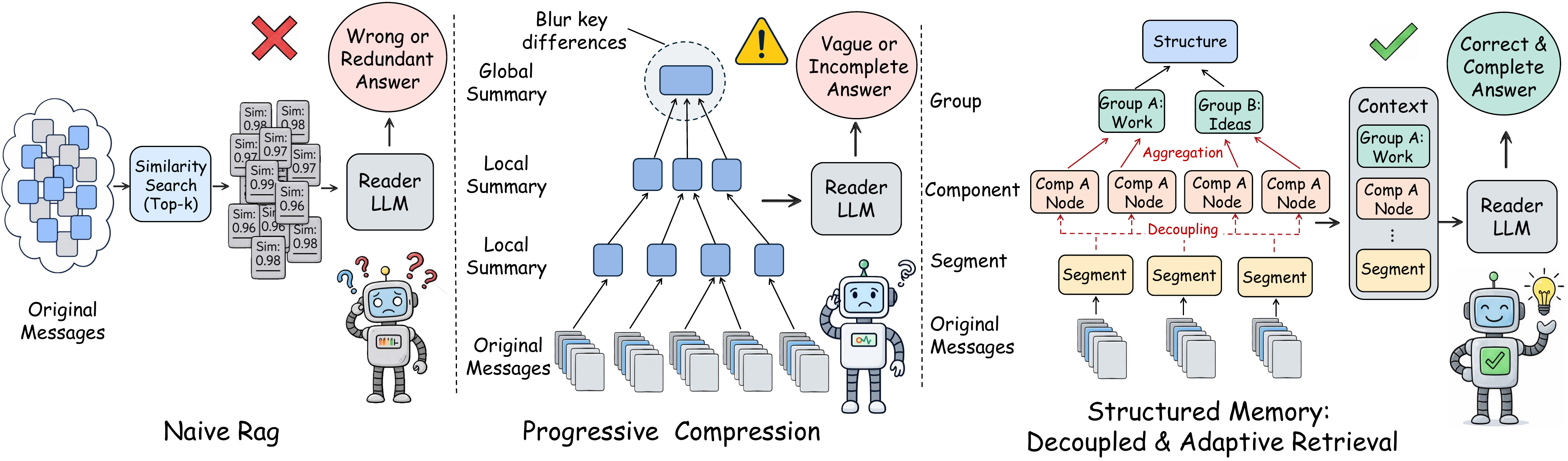}
\caption{\textbf{From similarity top-$k$ retrieval to evidence-oriented retrieval for agent memory.}
Agent memory forms a coherent and highly correlated interaction stream, in which many spans are near duplicates. As a result, fixed top-$k$ similarity retrieval often overweights redundant evidence instead of surfacing what is most useful for the query. \frameworkname first decouples similar histories into memory components that isolate distinctive evidence, and then organises them into a retrieval hierarchy that yields shorter, more sufficient context.}
\label{fig:introduction}
\end{figure*}

Large language model (LLM) agents increasingly rely on external memory to manage extended interaction histories, enabling capabilities such as multi-session dialogue, personalised assistance, and long-term task execution \cite{DBLP:conf/uist/ParkOCMLB23,tan-etal-2025-prospect,DBLP:journals/corr/abs-2412-13103}. In these settings, an agent answers new queries by retrieving useful evidence from memory systems rather than relying only on parametric knowledge \cite{DBLP:conf/aaai/ZhongGGYW24,DBLP:journals/corr/abs-2310-08560}. Current systems often approach memory retrieval as a standard RAG problem, relying on fixed top-$k$ similarity \cite{fang2026lightmem,kang-etal-2025-memory,xu2025amem}. 
However, agent memory poses a different retrieval challenge from standard RAG corpora. Unlike large heterogeneous document collections, agent memory forms a coherent interaction stream with substantial overlap in events and wording \cite{hu2025memory,NEURIPS2020_6b493230}.  Consequently, the central retrieval difficulty is no longer merely locating generally relevant text, but rather \emph{distinguishing decisive evidence hidden among highly similar histories}. 
As a result, similarity top-k retrieval often returns multiple memories that are generally relevant but largely repetitive. This is particularly problematic when the answer hinges on a small update, constraint, or factual difference among otherwise similar interactions.

This mismatch also exposes a limitation of existing memory organisation strategies. Recent systems use summaries, note structures, and hierarchical abstractions to improve scalability and navigation \cite{DBLP:journals/corr/abs-2501-13956,kang-etal-2025-memory,xu2025amem}. 
Although these structures reduce the cost of flat retrieval, many still organise memory into a progressively compressed hierarchy, where nodes become fewer and more compressed at higher levels.
Such designs can improve efficiency, but they do not directly address the need to identify answer-critical evidence from highly similar histories. Repeated compression often preserves shared background information while blurring the subtle details that distinguish one candidate from another.

Taken together, these considerations suggest a core design principle: \emph{decoupling before aggregation}. As shown in Fig. \ref{fig:introduction}, rather than treating long, highly similar interaction logs as whole retrieval units, the memory system should first decompose them into smaller evidence units that isolate reusable facts, state updates, and distinguishing details. 
Aggregation then organises these units into a higher-level memory structure that remains compact and coherent. Crucially, this structure should be revisable rather than fixed: as agent memory evolves, new interactions can reveal better relations among previously stored components and trigger corresponding updates to the high-level organisation.

Based on this principle, we propose \frameworkname, a framework that couples \textbf{memory structuring} with \textbf{adaptive retrieval}. 
Starting from raw messages, \frameworkname builds a hierarchical memory structure by following the principle of \emph{decoupling before aggregation}: it extracts \emph{segments} from raw messages as local events, decouples similar histories into \emph{memory components} that isolate reusable facts and distinctive attributes, and then aggregates related components into \emph{groups} for high-level access. 
To organise these components into useful groups, \frameworkname evaluates the component-to-group organisation with a guidance objective that balances sparsity and semantic faithfulness. 
This objective guides structure updates as memory evolves. 
When new memory arrives, components are attached to compatible groups when possible, and the structure is revised as needed by splitting overly large or internally heterogeneous groups and merging overly small or isolated ones. Retrieval finally proceeds from coarse units to fine ones: the system first selects complementary groups and components, and expands to segments and original messages only when more detailed evidence is needed.

Our contributions are summarised as follows:
\begin{enumerate}

\item We identify a key mismatch between standard RAG and agent memory, and argue that agent memory should follow the principle of \emph{decoupling before aggregation}.
    \item We propose \frameworkname, which constructs a revisable hierarchical memory structure by segmenting local interaction history, decoupling segments into memory components, and aggregating related components into high-level groups.
    \item We develop a top-down adaptive retrieval method over this memory structure, and show on LoCoMo and PerLTQA that it improves both answer quality and inference token efficiency.

\end{enumerate}
\section{Related Work}

\paragraph{RAG-style retrieval for agent memory.}
Many memory systems follow the standard RAG paradigm, storing past interactions as retrievable units and selecting a top-$k$ set by embedding similarity \cite{DBLP:journals/corr/abs-2310-08560, DBLP:journals/corr/abs-2602-19320}. This works well for heterogeneous corpora, but agent memory is often a coherent interaction stream with many semantically similar spans \cite{DBLP:journals/corr/abs-2512-13564}. As a result, similarity retrieval can return redundant memories while missing small updates or temporal distinctions that determine the answer. \frameworkname addresses the mismatch by retrieving decoupled evidence units rather than raw spans. 

\paragraph{Hierarchical and graph-based memory organisation.}
Recent systems organise memory into summaries, notes, temporal layers, semantic memories, or graph communities to improve scalability \cite{DBLP:journals/corr/abs-2602-05665,kang-etal-2025-memory,xu2025amem}. While these structures reduce the cost of flat retrieval, many remain summary- or schema-centric: higher-level nodes abstract over lower-level content instead of preserving fine-grained distinctions. \frameworkname instead follows \emph{decoupling before aggregation}: it first extracts reusable memory components from local segments, and then groups related components for efficient access.

\paragraph{Adaptive retrieval over dynamic memory structures.}
RAG-style memory systems usually treat past interactions as a static retrieval corpus, accessed by top-$k$ search, neighbour expansion, or multi-stage filtering~\cite{fang2026lightmem,DBLP:journals/corr/abs-2602-19320}. 
Some hierarchical systems support incremental updates by inserting new memories into existing clusters or nodes, but this remains distinct from retroactive reorganisation~\cite{DBLP:journals/corr/abs-2508-03341,li2026cam}. 
This distinction matters for agent memory, where later interactions can change how earlier evidence should be grouped. 
\frameworkname therefore maintains a revisable hierarchy that uses split and merge operations to reorganise previously stored components before adaptive retrieval.

\section{Method}
\label{sec:method}

\begin{figure*}[t]
\centering
\includegraphics[width=\linewidth]{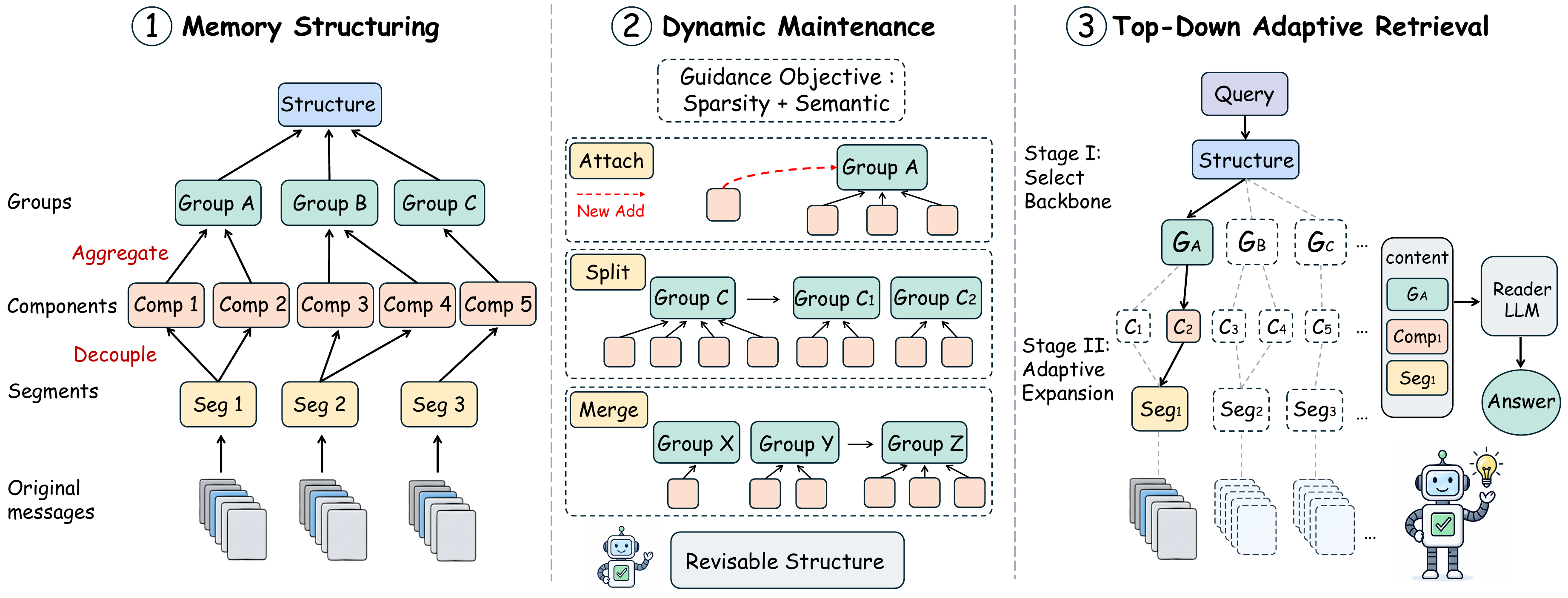}
\caption{\textbf{Overview of \frameworkname.}
\frameworkname follows the principle of \emph{decoupling before aggregation}. It first organises raw interaction history into local segments, then decomposes each segment into reusable memory components that isolate decisive evidence, and finally aggregates related components into high-level groups for efficient retrieval. At retrieval time, the system first selects a compact backbone of relevant groups and components, and expands to segments and raw messages only when additional textual evidence reduces reader uncertainty.}
\label{fig:methodology}
\end{figure*}

Given an interaction history $H=\{m_1,\dots,m_T\}$ and a query $q$, the goal is to construct a context $C$ for answering $q$. As shown in Figure~\ref{fig:methodology}, \frameworkname first organises the interaction history into a hierarchical memory structure. 
The structure is built in three stages: raw messages are divided into local \emph{segments}; each segment is then \emph{decoupled} into memory components that isolate reusable facts, constraints, and updates; and related components are finally \emph{aggregated} into higher-level groups for efficient access. Retrieval proceeds top-down over this structure, starting from relevant groups and components, and expanding to segments and original messages only when finer-grained textual evidence is needed.

\subsection{Building a Revisable Memory Structure}
We next describe how \frameworkname constructs the hierarchy and updates it as new memory arrives.

\paragraph{From messages to components.}
We first partition the message stream into contiguous \emph{segments}, each capturing a coherent local event, because answer-critical evidence in agent memory is often expressed across neighbouring turns rather than in isolated messages. Segments preserve this intact local context, but they are still too coarse to retrieve directly. We therefore extract one or more \emph{memory components} from each segment, where each component represents a reusable fact, constraint, attribute, relation, or state update. This is the \emph{decoupling} step of \frameworkname: it separates the decisive evidence from the surrounding local context before retrieval, while retaining links back to the source segment so that intact context can still be recovered when needed.

\paragraph{From components to groups.}
Once evidence has been decoupled into memory components, \frameworkname organises these components into higher-level \emph{groups}, which serve as compact access units for retrieval. This is the \emph{aggregation} step: rather than building the hierarchy directly over raw spans or summaries, we aggregate over already identified evidence units. 
The role of grouping is to make high-level retrieval both efficient and discriminative. If groups are too broad, retrieval becomes less discriminative and more redundant; if they are too fragmented, related evidence is scattered across many small units, weakening coverage for multi-fact queries. We therefore organise components into groups using an objective that balances \emph{sparsity} and \emph{semantic faithfulness}. 

Let $\mathcal{P}=\{C_k\}_{k=1}^{K}$ denote a partition of $N$ memory components into $K$ groups, where $n_k = |C_k|$. We score an organisation by
\begin{equation}
f(\mathcal{P}) = \mathrm{SparsityScore}(\mathcal{P}) + \mathrm{SemScore}(\mathcal{P}).
\label{eq:guidance}
\end{equation}
The sparsity term favours balanced groups by using the normalised inverse of the expected within-group scanning cost:
\begin{equation}
\mathrm{SparsityScore}(\mathcal{P})
= \frac{N^2}{K\sum_{k=1}^{K}n_k^2}.
\label{eq:sparsity}
\end{equation}
The semantic term encourages intra-group coherence and keeps the inter-group geometry navigable by Stage~I's kNN-based retrieval (Eq.~\eqref{eq:rep_select}).
\begin{equation}
\label{eq:semscore}
\mathrm{SemScore}(\mathcal{P})
= \frac{1}{K}\sum_{k=1}^{K}
\left(\frac{1}{n_k}\sum_{i\in C_k}\cos(\mathbf{x}_i,\boldsymbol{\mu}_k)\right)\cdot g(s_k),
\end{equation}
where $\mathbf{x}_i$ is the embedding of component $i$, $\boldsymbol{\mu}_k$ is the centroid of group $k$, and \(g(s_k)=\exp\!\left(-\frac{(s_k-\bar s)^2}{2\sigma^2}\right)\), with \(s_k=\max_{j\neq k}\cos(\boldsymbol{\mu}_k,\boldsymbol{\mu}_j)\), \(\bar s=\mathrm{median}(\{s_k\})\), and \(\sigma=\mathrm{median}(\{|s_k-\bar s|\})+\varepsilon\).
The bell-shaped $g(s_k)$ acts on inter-group geometry: near-duplicate centroids ($s_k$ above the median) reintroduce the redundancy that decoupling removes, while isolated ones ($s_k$ below the median) drop out of Stage~I's kNN expansion as ``semantic islands'' \cite{zhang2025leanrag}. Distinctiveness of individual components is already preserved upstream by decoupling; $g(s_k)$ only shapes group-centroid geometry.

\paragraph{Dynamic maintenance.}
Because agent memory evolves over time, the memory structure is maintained incrementally rather than fixed after construction. A new component with embedding $\mathbf{x}$ is attached to its nearest group \(k^\star = \arg\max_{1 \le k \le K} \cos(\mathbf{x}, \boldsymbol{\mu}_k)\) if the similarity exceeds a threshold; otherwise a new group is created. 
Groups that become too large or internally incoherent are considered for splitting, while small or isolated groups are considered for merging. In both cases, the selected operation is the one that most improves Eq.~\eqref{eq:guidance}.

We additionally maintain top-$k$ nearest-neighbour links among groups to support local retrieval. Further details and a qualitative example are provided in Appendices~\ref{app:construction}, \ref{app:structure}, and~\ref{app:case_study}.

\subsection{Retrieval from the Memory Structure}

The memory structure separates two roles that raw interaction history entangles: groups and memory components expose high-level evidence, while segments and original messages preserve the intact local context in which that evidence was expressed. Retrieval therefore proceeds in two stages. We first identify a compact high-level backbone of relevant evidence, and only then expand to lower-level text when more context is necessary for answer generation.

\paragraph{Stage I: high-level backbone selection.}
Answers in agent memory are often supported by multiple related evidence units rather than a single best match. Stage~I therefore selects a compact set of complementary high-level units, rather than repeatedly retrieving near-duplicate evidence by pure similarity ranking.
To support this, we maintain kNN links among groups in the memory structure. These links record which high-level units are semantically close, even when they belong to different groups, allowing Stage I to avoid repeatedly selecting near-duplicate evidence and to identify related evidence regions that remain uncovered.

 We first retrieve a small pool of groups by similarity between query and centroid, and include their neighbouring groups and associated components as candidate set $V$. Let $R \subseteq V$ denote the selected evidence backbone.
 For each candidate node $i \in V$, let $\mathcal{N}(i)$ denote its neighbours under the kNN links, and let $w_{iu}>0$ denote the similarity weight on edge $(i,u)$. We define the covered set as \(C(R)=\{u\in V \mid \exists r\in R,\; u\in \{r\}\cup \mathcal{N}(r)\}\), and the newly covered nodes contributed by candidate $i$ as \(\Delta(i;R)=(\{i\}\cup \mathcal{N}(i))\setminus C(R)\). At each step, we greedily select the next unit by trading off structural coverage against query relevance:
\begin{equation}
i^\star = \arg\max_{i\in V\setminus R}
\left[ \frac{\sum_{u\in \Delta(i;R)} w_{iu}}{Z}
+ \tilde{s}(q,i)
\right],
\label{eq:rep_select}
\end{equation}
where $\tilde{s}(q,i)\in[0,1]$ is the normalised query--node similarity, and $Z$ denotes the total candidate coverage weight. We apply Eq.~\eqref{eq:rep_select} hierarchically,  first selecting groups and then refining to components.

\paragraph{Stage II: adaptive text expansion.}
The high-level backbone indicates where useful evidence resides, but it does not yet determine how much original text should be revealed to the reader. Expanding all linked segments and messages would often reintroduce the same redundancy that the high-level structure is designed to avoid. Stage~II therefore adds lower-level text only when it further reduces the reader's uncertainty.

Starting from the selected components, we gather their linked segments and construct a coarse context. Let $U(C,q)$ denote the reader's uncertainty under context $C$ and query $q$, instantiated in our implementation via predictive entropy; concrete definitions and implementation details are given in Appendix~\ref{app:retrieval}. For a candidate segment $s$, its marginal uncertainty reduction is
\begin{equation}
\Delta U(s \mid C,q)=U(C,q)-U(C\cup\{s\},q).
\label{eq:uncertainty_gain}
\end{equation}
A segment is included only if it yields a positive marginal reduction, i.e., $\Delta U(s \mid C,q)>0$. For a candidate message $m$, we analogously compute \(\Delta U(m \mid C,q)=U(C,q)-U(C\cup\{m\},q)\)
and include it only when it further reduces uncertainty. Retrieval stops when no remaining candidate segment or message provides additional uncertainty reduction.

\section{Experiments}

\subsection{Experimental Setup}
\label{sec:experiment setup}
\paragraph{Datasets and metrics.}
We evaluate long-term agent memory on two complementary benchmarks: LoCoMo~\cite{DBLP:conf/acl/MaharanaLTBBF24} and PerLTQA~\cite{du-etal-2024-perltqa}. LoCoMo contains 50 multi-session dialogues, with an average of $\sim$18K tokens and $\sim$300 turns. Following prior work, we report results on its four answerable categories and omit the adversarial subset, which does not provide gold answers for the metrics used here. PerLTQA evaluates personalised long-term memory over longer contexts (around 25K tokens on average), with answers that are often sentence-style rather than short spans. We report BLEU-1~\cite{papineni-etal-2002-bleu} and token-level F1 on both datasets, and additionally ROUGE-L~\cite{lin-2004-rouge} on PerLTQA to better capture sequence-level overlap for longer-form answers.

\paragraph{Baselines.}
We compare against six baselines, grouped by how they access and organise memory. These baselines cover full-context reading, flat similarity retrieval, post-retrieval compression, and representative structured-memory designs. \textbf{Flat or minimally structured baselines} include \textbf{(1) Full Memory}, which provides the complete available history to the reader model without retrieval; \textbf{(2) Naive RAG}, which chunks original messages and retrieves the top-$20$ chunks by vector similarity; and \textbf{(3) LightMem}, which reduces retrieval cost through multi-stage filtering and compression~\cite{fang2026lightmem}. \textbf{Structured memory baselines} include \textbf{(4) Nemori}, which builds a hierarchical memory from episodic memories to higher-level semantic memories for long-term recall~\cite{DBLP:journals/corr/abs-2508-03341}; \textbf{(5) A-Mem}, which stores memory as structured notes connected by dynamic links~\cite{xu2025amem}; and \textbf{(6) MemoryOS}, which organises memory into temporally layered storage with lifecycle management~\cite{kang-etal-2025-memory}.

\paragraph{Implementation details.}
\label{sec:implementation_details}
We evaluate all compared methods with three reader LLMs: two recent open-source models, Qwen3-8B~\cite{DBLP:journals/corr/abs-2505-09388} and Llama-3.1-8B-Instruct~\cite{DBLP:journals/corr/abs-2407-21783}, and one closed-source model, GPT-5 nano. Final answers are generated with greedy decoding (temperature $=0.0$) for deterministic evaluation~\cite{DBLP:conf/emnlp/0001YMWRCYR23}. All retrieval and memory construction embeddings use \texttt{text-embedding-3-small}, and each method uses the same backbone model for memory construction and answer generation. Since LoCoMo and PerLTQA require different answer formats, we use dataset-specific answer prompts 
within each dataset. For GPT-5 nano, which does not expose token-level logits, we estimate uncertainty with GPT-4.1-mini. All xMemory hyperparameters are fixed across datasets and backbone models; their values and implementation details are reported in Appendix~\ref{app:implementation_details} for reproducibility.

\subsection{Main Results}

\newcommand{\flatmark}{\textsuperscript{\textsc{F}}}
\newcommand{\structmark}{\textsuperscript{\textsc{S}}}

\begin{table*}[t]
\centering
\small
\setlength{\tabcolsep}{3.0pt}
\renewcommand{\arraystretch}{1.12}
\caption{Main results on LoCoMo. Methods marked with \flatmark{} are flat retrieval baselines, and methods marked with \structmark{} are structured memory baselines. We report BLEU and F1 for each question category and the average. Token/query denotes the average total tokens per query (lower is better) during inference time. Best results within each backbone model are in bold.}
\label{tab:locomo_main}
\resizebox{ 0.95\textwidth}{!}{%
\begin{tabular}{l|lrrrrrrrrrrr}
\toprule
\textbf{Model} & \textbf{Method} &
\multicolumn{2}{c}{\textbf{Multi-hop}} &
\multicolumn{2}{c}{\textbf{Temporal}} &
\multicolumn{2}{c}{\textbf{Open-domain}} &
\multicolumn{2}{c}{\textbf{Single-hop}} &
\multicolumn{2}{c}{\textbf{Average}} &
\multicolumn{1}{c}{\makecell{\textbf{Token}\\\textbf{/query}}} \\
\cmidrule(lr){3-4}\cmidrule(lr){5-6}\cmidrule(lr){7-8}\cmidrule(lr){9-10}\cmidrule(lr){11-12}
& &
\textbf{BLEU} & \textbf{F1} &
\textbf{BLEU} & \textbf{F1} &
\textbf{BLEU} & \textbf{F1} &
\textbf{BLEU} & \textbf{F1} &
\textbf{BLEU} & \textbf{F1} & \\
\midrule

\multirow{7}{*}{\rotatebox{90}{Qwen3-8B}}
& Full Memory\flatmark   & 24.63 & 31.66 & 10.02 & 12.58 & 12.48 & 17.59 & 33.80 & 44.00 & 25.83 & 33.54 & 18535.90 \\
& Naive RAG\flatmark     & 22.46 & 34.28 & 17.21 & 21.32 & 12.35 & 17.08 & 35.68 & 45.22 & 27.95 & 36.48 & 8633.28 \\
& LightMem\flatmark      & 19.63 & 26.23 & 22.67 & 27.61 & 9.83 & 14.37 & 31.34 & 40.66 & 26.04 & 33.66 & 5545.35 \\
\cmidrule(lr){2-13}
& Nemori\structmark      & 24.68 & 36.82 & 25.78 & 33.76 & 12.71 & 18.48 & 38.02 & 47.52	 &31.44 & 40.88 & 7754.66 \\
& A-Mem\structmark       & 23.32 & 33.26 & 22.29 & 32.53 &9.08  & 16.72 &33.61 & 42.52 &27.84 & 37.13 & 9103.46 \\
& MemoryOS\structmark    & 17.12 & 21.74 & 26.84 & 32.26 & 14.53 & 16.25 & 35.83 & 40.37 & 29.20 & 33.76 & 7234.66 \\
\cmidrule(lr){2-13}
& \frameworkname\textbf{ (Ours)} & \textbf{27.24} & \textbf{38.57} & \textbf{29.58} & \textbf{37.46} & \textbf{15.55} & \textbf{20.69} & \textbf{40.94} & \textbf{50.94} & \textbf{34.48} & \textbf{43.98} & \textbf{4711.29} \\

\midrule

\multirow{7}{*}{\rotatebox{90}{Llama-3.1-8B-Ins}}
& Full Memory\flatmark   & 15.11 & 21.52 & 6.06 & 8.06 & 7.78 & 11.13 & 23.29 & 35.13 & 17.23 & 25.50 & 18524.70 \\
& Naive RAG\flatmark     & 13.77 & 19.40 & 6.20 & 8.61 & 6.99 & 10.22 & 21.90 & 35.37 & 16.21 & 25.30 & 11522.60 \\
& LightMem\flatmark      & 16.43 & 23.76 & 12.44 & 15.61 & 9.83 & 13.70 & 21.68 & 25.08 & 18.05 & 22.15 & 5708.11 \\
\cmidrule(lr){2-13}
& Nemori\structmark      & 18.18 & 26.30 & 19.23 & 26.09 & 9.24 & 12.06 & 26.19 & 40.18 & 22.21 & 32.95 & 9802.69 \\
& A-Mem\structmark       & 15.88 & 20.90 & 19.92 & 24.91 & 9.92 & 11.34 & 24.86 & 39.90 & 21.52 & 31.51 & 10268.77 \\
& MemoryOS\structmark    & 13.26 & 17.83 & 19.88 & 23.95 & 11.48 & 12.92 & 24.76 & 28.67 & 20.81 & 24.72 & 7212.07 \\
\cmidrule(lr){2-13}
& \frameworkname\textbf{ (Ours)} & \textbf{22.21} & \textbf{30.99} & \textbf{21.20} & \textbf{27.42} & \textbf{11.58} & \textbf{14.61} & \textbf{28.43} & \textbf{41.15} & \textbf{24.73} & \textbf{34.77} & \textbf{5539.97} \\

\midrule

\multirow{7}{*}{\rotatebox{90}{GPT-5 nano}}
& Full Memory\flatmark   & 22.95 & 34.07 & 20.02 & 23.50 & 21.61 & 26.42 & 32.71 & 46.97 & 27.58 & 38.44 & 18544.25 \\
& Naive RAG\flatmark     & 21.34 & 32.31 & 28.43 & 29.05 & 22.63 & 26.56 & 36.68 & 46.35 & 31.28 & 38.94 & 7531.46 \\
& LightMem\flatmark      & 23.13 & 32.07 & 41.06 & 55.23 & 21.27 & 26.38 & 37.52 &44.31  & 34.60 & 43.23 & 6850.04 \\
\cmidrule(lr){2-13}
& Nemori\structmark      & 24.80 & 37.11 & 41.56 & 54.25 & 22.61 & 29.29 & 40.35 & 51.67 & 36.65 & 48.17 & 9154.76 \\
& A-Mem\structmark       & 24.12& 35.88 & 40.87 & 54.67 & 20.83 & 27.61 & 38.42 & 46.13 & 35.22 & 44.88 & 9610.94 \\
& MemoryOS\structmark    & 24.47 & 36.13 & 39.78 & 55.34 & 21.45 &28.71 & 39.39 & 48.68 & 35.62 & 46.53 & 7029.02 \\
\cmidrule(lr){2-13}
& \frameworkname\textbf{ (Ours)} & \textbf{27.56} & \textbf{39.97} & \textbf{46.10} & \textbf{57.62} & \textbf{25.53} & \textbf{30.92} & \textbf{41.14} & \textbf{52.63} & \textbf{38.71} & \textbf{50.00} & \textbf{6581.20} \\

\bottomrule
\end{tabular}%
}
\end{table*}

\paragraph{LoCoMo: retrieval over coherent multi-session histories.}
Table~\ref{tab:locomo_main} reports results on LoCoMo. Across all three backbones, \frameworkname achieves the best average performance, with especially clear gains on multi-hop and temporal questions. For example, with Qwen3-8B, average BLEU/F1 improves from 31.44/40.88 for Nemori to \textbf{34.48}/\textbf{43.98}; with GPT-5 nano, it improves from 36.65/48.17 to \textbf{38.71}/\textbf{50.00} while reducing token usage from 9155 to \textbf{6581} per query.

Compared with both flat retrieval and structured memory baselines, these results suggest that agent memory benefits from decoupling decisive evidence before high-level organisation and retrieval. Full Memory and Naive RAG preserve large amounts of raw context but often return redundant or weakly discriminative evidence, while LightMem reduces cost mainly through post-retrieval compression and performs worse on multi-hop and temporal questions. Structured baselines introduce higher-level organisation, but their memory structures are not explicitly centred on decoupled evidence units and revisable grouping. By contrast, \frameworkname first decouples distinctive evidence into memory components, organises them into revisable groups, and retrieves top-down over this structure.

\begin{table}[t]
\centering
\small
\setlength{\tabcolsep}{3.8pt}
\renewcommand{\arraystretch}{1.12}
\caption{Main results on PerLTQA. Methods marked with \flatmark{} are flat retrieval baselines, and methods marked with \structmark{} are structured memory baselines. We report BLEU, F1, R-L (ROUGE-L), and Tok.\ (token usage per query during inference time, lower is better). Best results within each backbone model are in bold.}
\label{tab:perltqa_main}
\resizebox{0.9 \textwidth}{!}{%
\begin{tabular}{l|cccc|cccc|cccc}
\toprule
\multirow{2}{*}{\textbf{Method}}
& \multicolumn{4}{c|}{\textbf{Qwen3-8B}}
& \multicolumn{4}{c|}{\textbf{Llama-3.1-8B-Ins}}
& \multicolumn{4}{c}{\textbf{GPT-5 nano}} \\
& BLEU & F1 & R-L & Tok.
& BLEU & F1 & R-L & Tok.
& BLEU & F1 & R-L & Tok. \\
\midrule

Full Memory\flatmark
& 32.73 & 42.15 & 36.65 & 25045
& 33.00 & 44.40 & 39.43 & 25098
& 24.37 & 34.97 & 29.81 & 25212 \\
Naive RAG\flatmark
& 32.08 & 41.37 & 35.95 & 6274
& 33.67 & 44.84 & 39.65 & 9531
& 27.35 & 37.76 & 32.35 & 10756 \\
LightMem\flatmark
& 29.12		 & 40.21 & 34.67 & 7692
& 31.83 & 42.33 & 38.37 & \textbf{5452}
& 24.84 & 34.51 & 29.33 & 7579 \\
\midrule

Nemori\structmark
& 32.55 & 42.80 & 38.05 & 9092
& 41.01 & 49.62 & 44.65 & 11440
& 33.44 & 41.79 & 38.43 & 11883 \\
A-Mem\structmark
& 31.36 & 40.92 & 36.45 & 9864
& 35.56		 & 45.79 & 41.38 & 7707
& 33.12 & 41.17 & 37.94 & 14718 \\
MemoryOS\structmark
& 35.14 & 42.35 & 38.48 & 6499
& 34.79		 & 42.03 & 38.22 & 6511
& 27.66 & 33.83 & 31.44 & 12669 \\

\midrule

\frameworkname\textbf{ (Ours)}
& \textbf{36.24} & \textbf{47.08} & \textbf{42.50} & \textbf{5087}
& \textbf{42.68} & \textbf{52.37} & \textbf{47.84} & 6066
& \textbf{36.79} & \textbf{46.23} & \textbf{41.25} & \textbf{7307} \\

\bottomrule
\end{tabular}%
}
\end{table}

\paragraph{PerLTQA: generalisation to longer personalised memory.}
Table~\ref{tab:perltqa_main} reports results on PerLTQA, which contains longer contexts and more sentence-style answers than LoCoMo. \frameworkname remains consistently effective across all three backbone models, showing that the proposed retrieval principle transfers beyond multi-session dialogue recall to longer personalised memory reasoning. With Qwen3-8B, it achieves the best BLEU/F1/ROUGE-L at \textbf{36.24}/\textbf{47.08}/\textbf{42.50} while using the fewest tokens; with Llama-3.1-8B-Instruct, it reaches \textbf{42.68}/\textbf{52.37}/\textbf{47.84}. 
These gains suggest that \frameworkname improves selective evidence access while preserving the information needed for coherent sentence-level reconstruction.

\paragraph{Efficiency.}
We report \emph{tokens per query} as the average end-to-end inference cost, including retrieval, answer generation, and auxiliary calls. 
Across both datasets, \frameworkname achieves better answer quality with competitive or fewer tokens than strong memory baselines, indicating that its gains come from delivering more concentrated evidence rather than exposing more history. 
For example, on LoCoMo with Qwen3-8B, \frameworkname reduces token usage from 7755 for Nemori to \textbf{4711} while also improving average BLEU/F1 from 31.44/40.88 to \textbf{34.48}/\textbf{43.98}. These results suggest that token efficiency in agent memory should be measured by evidence utility rather than context reduction alone.

\section{Analysis}

\subsection{Ablation Studies}

\paragraph{Retrieval stage analysis.}
We first ablate memory structuring and the two retrieval stages on LoCoMo with Qwen3-8B. 
As shown in Figure~\ref{fig:ablation_summary} (a), Memory-only improves average BLEU/F1 from 27.95/36.48 for Naive RAG to 31.81/40.77, showing that retrieval over decoupled memory units is more effective than retrieval over flat raw chunks, even with basic similarity matching. 
Adding Stage~I improves selection over groups and components, while Stage~II improves uncertainty guided expansion to segments and messages. 
Combining both stages gives the best trade-off, with the highest average BLEU/F1 and the lowest token usage. 

\paragraph{Group size upper bound analysis.}
We next study the upper bound on the number of memory components per group. 
This parameter controls routing arity: a larger candidate set within each group makes decisive evidence harder to identify in dialogue memory with high semantic similarity, while an overly small group size scatters related facts. 
Guided by a Fano style lower bound, which shows that routing error increases with candidate size when discriminative information is bounded, we choose 12 as a practical threshold and then validate it empirically. 
Figure~\ref{fig:ablation_summary} (b) shows that this setting performs best, reaching 34.48 BLEU and 43.98 F1 with 4.48 components per group on average. 
Larger groups weaken discrimination, while smaller groups fragment the hierarchy. 
Full results and the detailed theoretical motivation are provided in Appendix~\ref{app:ablation_details} and Appendix~\ref{app:fano_motivation}.

\begin{figure}[t]
\centering
\includegraphics[width=\linewidth]{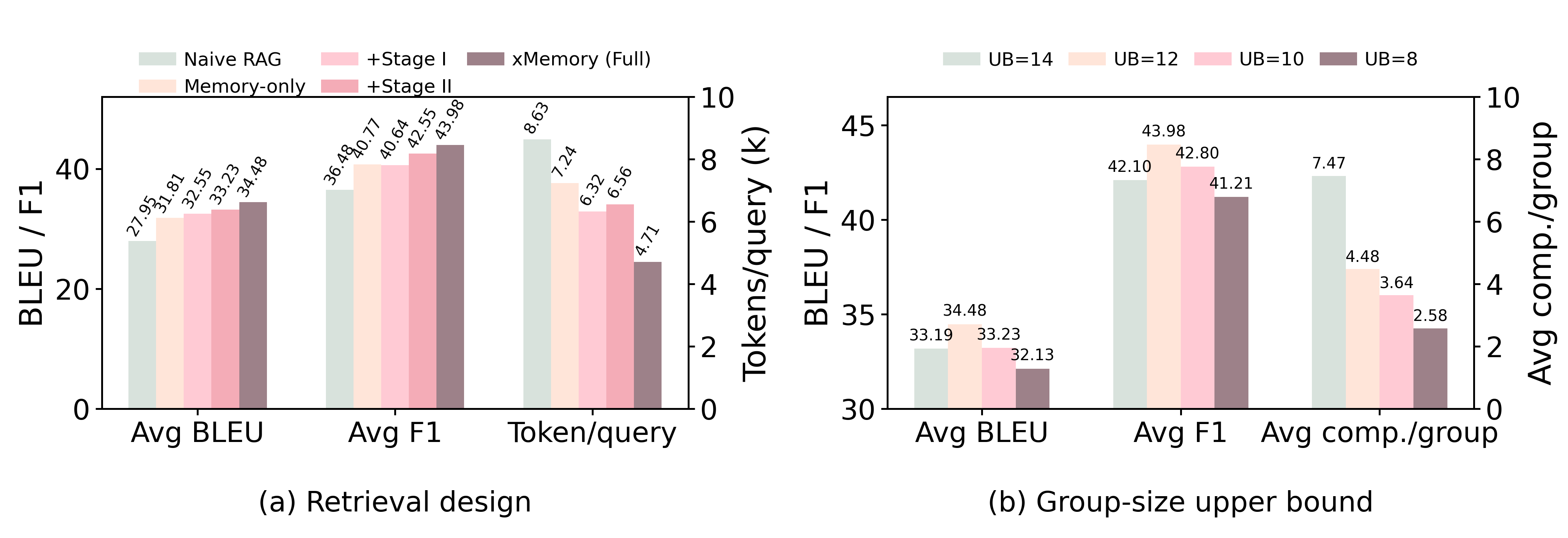}
\caption{\textbf{Ablation studies on LoCoMo with Qwen3-8B.}
(a) Retrieval-stage ablation summarising the effect of memory structuring, Stage~I, and Stage~II on average BLEU/F1 and token usage. 
(b) Effect of the upper bound (UB) of memory components per group, showing average BLEU/F1 and the resulting average number of components per group.
}
\label{fig:ablation_summary}
\end{figure}

\subsection{Retroactive Restructuring in Memory Construction}
\label{sec:analysis:reassign}

Unlike standard RAG, where the retrieval corpus is usually fixed, agent memory evolves as new interactions arrive. 
In \frameworkname, later insertions can trigger \textit{split} or \textit{merge} operations over high-level groups, thereby revising the assignment of previously created memory components. 
We measure this effect with the \textbf{dynamic reassignment ratio}, defined as the fraction of existing components whose group assignment changes during later insertions.

\begin{wrapfigure}{r}{0.5\linewidth}
    \vspace{-1.0em}
    \centering
    \includegraphics[width=\linewidth]{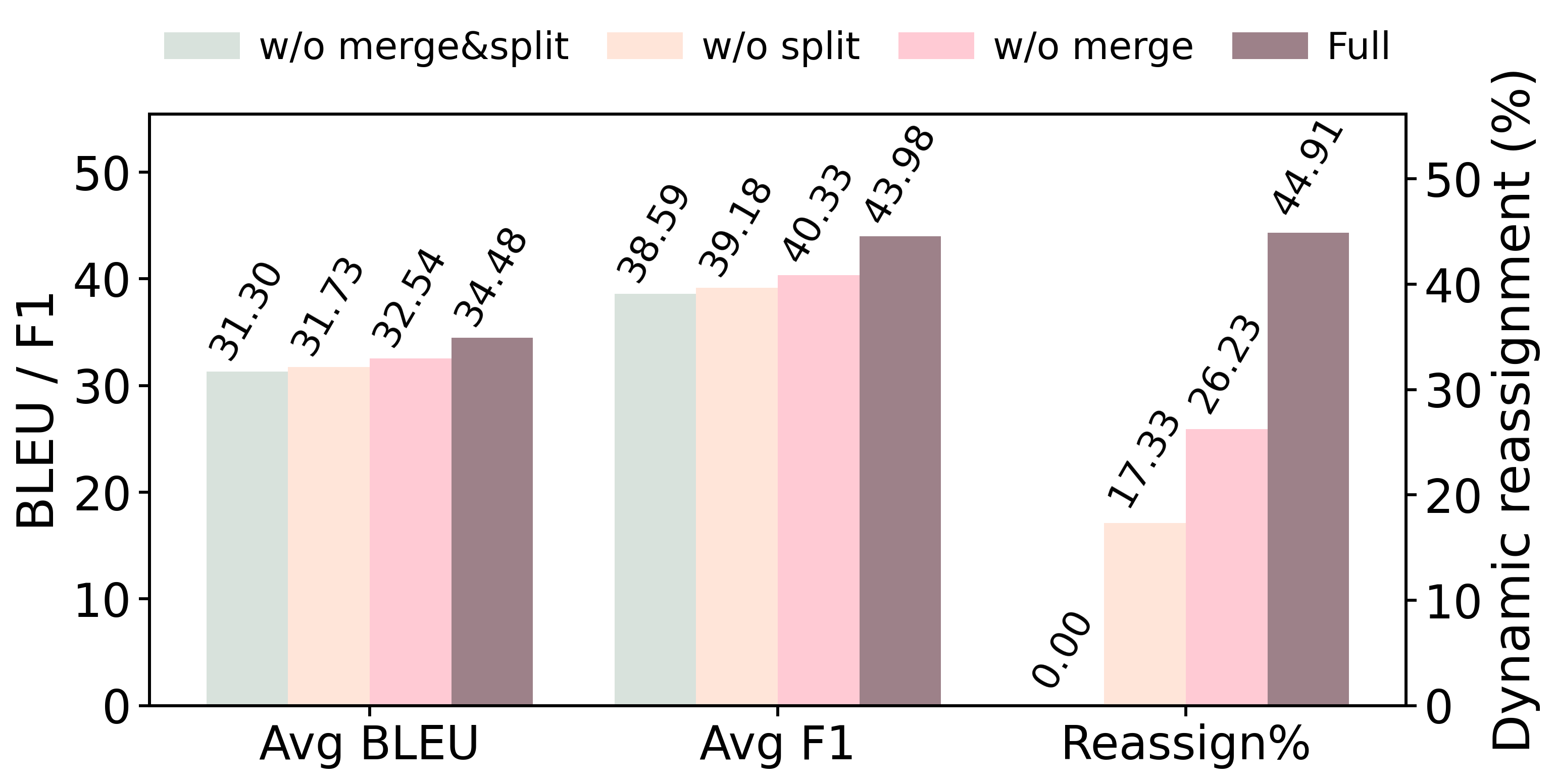}
    \vspace{-0.8em}
\caption{\textbf{Retroactive restructuring.}
Freezing high-level restructuring yields 0\% reassignment and lower QA performance, while the full system performs substantial reassignment and achieves the best Avg.\ BLEU/F1.}
    \label{fig:retroactive_restructuring}
    \vspace{-2.0em}
\end{wrapfigure}

Figure~\ref{fig:retroactive_restructuring} shows that such retroactive restructuring improves downstream QA. 
Disabling both \textit{split} and \textit{merge} freezes the structure, yields 0\% reassignment, and reduces average F1 to 38.59, whereas the full system reaches the highest reassignment ratio (44.91\%) and the best average F1 (43.98). 
The structure statistics in Appendix~\ref{app:reassign_details} further show that \textit{split} and \textit{merge} play complementary roles: \textit{split} enables revision by repartitioning broad groups, while \textit{merge} prevents the high-level index from becoming unnecessarily fragmented. 
These results support our claim that revisability is not merely an implementation detail, but a useful property for memory organisation under evolving evidence.

\subsection{Retrieval Efficiency and Cost Performance Trade-off}
\label{sec:analysis:efficiency_tradeoff}

\begin{figure}[t]
    \centering
    \includegraphics[width=\linewidth]{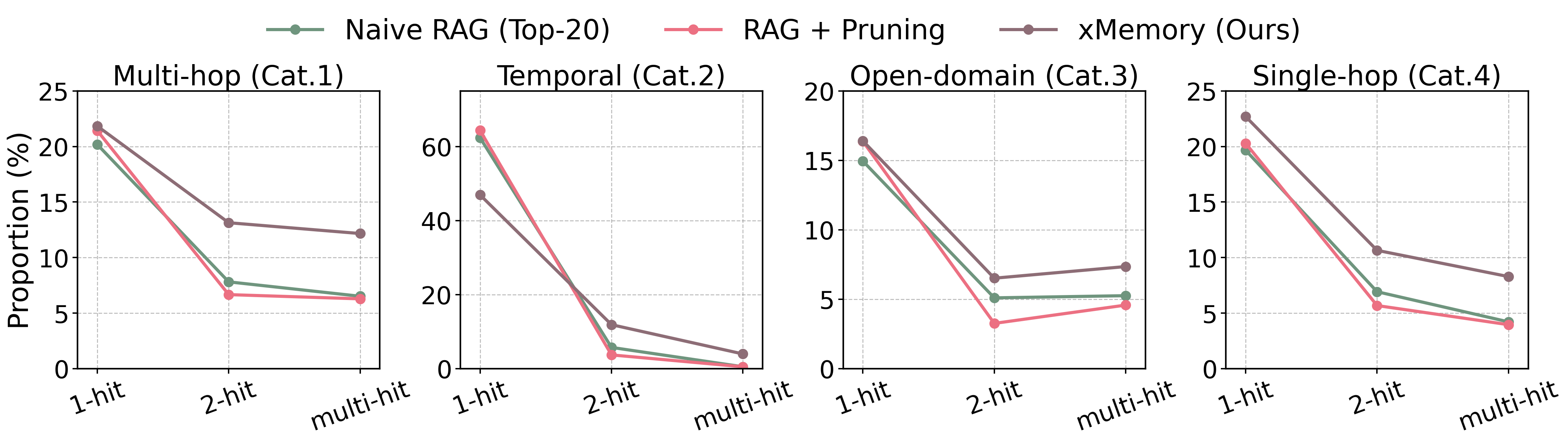}
\caption{Evidence hit distribution by question category. Each subplot shows the proportions of retrieved blocks that are 1-hit, 2-hit, or multi-hit, comparing Naive RAG, RAG with pruning, and \frameworkname. Higher 2-hit and multi-hit proportions indicate denser evidence retrieval.}
    \label{fig:hit-ratio-by-category}
\end{figure}

\paragraph{Evidence density and coverage efficiency.}
To understand why \frameworkname improves answer quality while reducing inference-time context cost, we compare three retrieval settings on LoCoMo with Qwen3-8B: (1) Naive RAG with top-$20$ chunks; (2) RAG with the LLMLingua-2 pruning module~\cite{pan-etal-2024-llmlingua} as in LightMem~\cite{fang2026lightmem}; and (3) \frameworkname. For each query, we remove stopwords from the reference answer and treat the remaining content words as answer evidence units. A retrieved block is counted as \textbf{1-hit}, \textbf{2-hit}, or \textbf{multi-hit} if it contains one, two, or at least three distinct answer evidence units, respectively. Figure~\ref{fig:hit-ratio-by-category} shows that \frameworkname retrieves denser evidence across all question categories, with consistently higher 2-hit and multi-hit proportions than both RAG baselines. This suggests that decoupling memory into finer-grained components helps concentrate answer-relevant evidence before retrieval, rather than leaving it dispersed across redundant raw chunks. By contrast, pruning shifts mass from 2-hit and multi-hit blocks toward 1-hit blocks, suggesting that post-retrieval compression can fragment entangled evidence even when the remaining context is broadly relevant. For example, on multi-hop questions, \frameworkname increases the 2-hit and multi-hit proportions to 13.14\% and 12.19\%, compared with 7.82\% and 6.53\% for Naive RAG.

\begin{wraptable}{r}{0.5\textwidth}
\centering
\vspace{-1.3em}
\small
\setlength{\tabcolsep}{3.8pt}
\renewcommand{\arraystretch}{1.10}
\caption{\textbf{Average evidence coverage efficiency on LoCoMo.}
Blocks denotes the average number of retrieved blocks needed to cover all answer evidence units, and Tokens denotes the corresponding token cost. Full category-level results are in Appendix~\ref{app:rag effect}.}
\label{tab:evidence_coverage_avg}
\resizebox{\linewidth}{!}{
\begin{tabular}{lrrrr}
\toprule
\textbf{Method} & \textbf{BLEU} & \textbf{F1} & \textbf{Blocks}$\downarrow$ & \textbf{Tokens}$\downarrow$ \\
\midrule
Naive RAG     & 27.95 & 36.48 & 10.81 & 1979.26 \\
RAG + Pruning & 26.55 & 34.58 & 13.31 & 1587.99 \\
\frameworkname & \textbf{34.48} & \textbf{43.98} & \textbf{5.66} & \textbf{974.56} \\
\bottomrule
\end{tabular}
}
\vspace{-1.0em}
\end{wraptable}

We further evaluate \emph{coverage efficiency} by comparing how many retrieved blocks and tokens are needed to cover all answer evidence units. Table~\ref{tab:evidence_coverage_avg} shows that pruning reduces token cost relative to Naive RAG, but requires more blocks and yields lower accuracy, consistent with fragmented evidence after compression. In contrast, \frameworkname achieves the best BLEU/F1 while covering answer evidence with substantially fewer blocks and tokens. This suggests that \frameworkname improves not simply by retrieving less context, but by concentrating more answer-relevant evidence into a smaller retrieval budget.

\begin{wrapfigure}{r}{0.5\textwidth}
    \centering
    \vspace{-1.0em}
    \includegraphics[width=\linewidth]{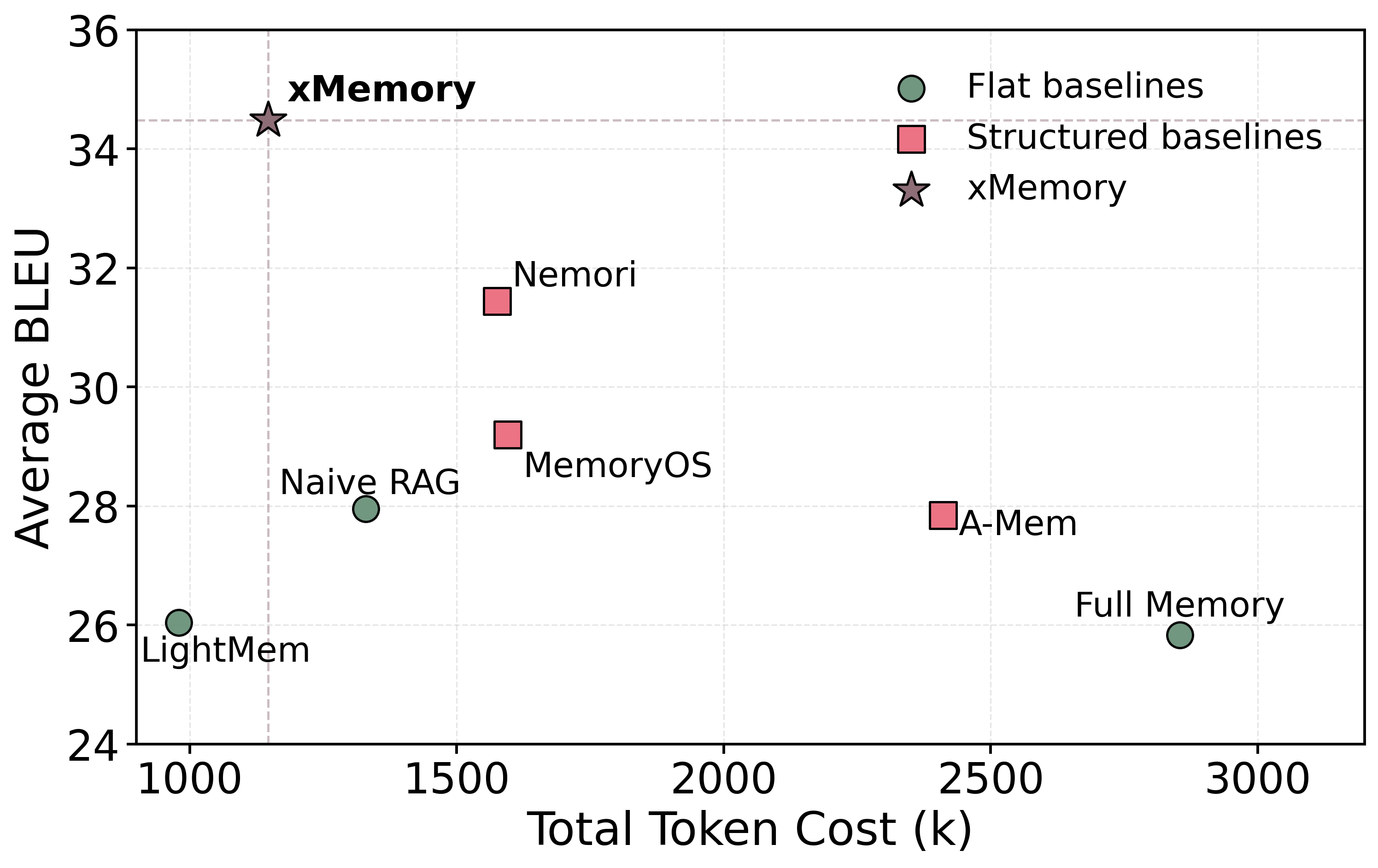}
    \caption{\textbf{Cost--performance trade-off on LoCoMo.}
    Total token cost is plotted against average BLEU with Qwen3-8B.}
    \label{fig:locomo_cost_tradeoff}
    \vspace{-1.0em}
\end{wrapfigure}

\paragraph{End-to-end cost--performance trade-off.}
The above analysis focuses on inference-time retrieval efficiency. We next examine whether this advantage remains after accounting for memory construction cost. Following the LightMem-style reporting setup~\cite{fang2026lightmem}, we calculate construction tokens and compute total token cost as average construction cost plus average inference tokens per query multiplied by the corresponding query count. The exact statistics used for this analysis are reported in Appendix~\ref{app:cost_perf_stats}.

As shown in Figure~\ref{fig:locomo_cost_tradeoff}, \frameworkname achieves the most favourable trade-off. Flat baselines such as Naive RAG and LightMem use relatively low or moderate token budgets, but obtain lower answer quality. Structured memory baselines improve over flat retrieval in some cases, but usually incur higher construction and inference costs. In contrast, \frameworkname attains the highest average BLEU while using a lower total token cost than Nemori, MemoryOS, and A-Mem. This suggests that the denser evidence retrieved by \frameworkname not only improves local coverage efficiency, but also translates into a better end-to-end efficiency--quality balance after amortising construction cost.

\section{Conclusion}
We presented \frameworkname, a retrieval framework that uses hierarchical organisation to address the
mismatch between agent memory and standard RAG assumptions. By disentangling highly correlated
memories into semantic components and retrieving top-down over the resulting hierarchy, xMemory
reduces redundancy while preserving temporally linked evidence. Across LoCoMo and PerLTQA,
xMemory improves answer quality with lower token cost and retrieves more evidence-dense contexts
than RAG baselines. These results suggest that effective agent memory should not only retrieve or
compress past interactions, but also reorganise them into evidence-oriented structures that make
subtle distinctions easier to access. We hope xMemory provides a useful step toward more adaptive
and efficient memory systems for long-horizon LLM agents.

\bibliographystyle{unsrtnat}
\bibliography{reference}

\clearpage

\newpage

\appendix

\hypersetup{linkcolor=black}
\section*{Appendix}
\noindent
\textbf{\hyperref[app:method]{A\quad Additional Method Details}} \dotfill \pageref{app:method}\\[2pt]
\hspace*{1.5em}A.1\quad \hyperref[app:construction]{From Raw Messages to Segments and Components} \dotfill \pageref{app:construction}\\[2pt]
\hspace*{1.5em}A.2\quad \hyperref[app:structure]{Detailed Structure Maintenance} \dotfill \pageref{app:structure}\\[2pt]
\hspace*{1.5em}A.3\quad \hyperref[app:retrieval]{Detailed Retrieval Implementation} \dotfill \pageref{app:retrieval}\\[2pt]
\hspace*{1.5em}A.4\quad \hyperref[app:pipeline]{Pseudo-code of the Full Pipeline} \dotfill \pageref{app:pipeline}\\[4pt]
\textbf{\hyperref[app:implementation_details]{B\quad Implementation Setup}} \dotfill \pageref{app:implementation_details}\\[4pt]
\textbf{C\quad More Experiment Results} \dotfill \pageref{app:ablation_details}\\[2pt]
\hspace*{1.5em}C.1\quad \hyperref[app:ablation_details]{Detailed Ablation Results} \dotfill \pageref{app:ablation_details}\\[2pt]
\hspace*{1.5em}C.2\quad \hyperref[app:fano_motivation]{Theoretical Motivation of Branching Factor} \dotfill \pageref{app:fano_motivation}\\[2pt]
\hspace*{1.5em}C.3\quad \hyperref[app:rag effect]{Evidence Density under RAG and Pruning Effect} \dotfill \pageref{app:rag effect}\\[2pt]
\hspace*{1.5em}C.4\quad \hyperref[app:reassign_details]{Retroactive Restructuring in Memory Construction} \dotfill \pageref{app:reassign_details}\\[2pt]
\hspace*{1.5em}C.5\quad \hyperref[app:cost_perf_stats]{Additional Cost--Performance Statistics} \dotfill \pageref{app:cost_perf_stats}\\[4pt]
\textbf{\hyperref[app:case_study]{D\quad Case Study: Decoupling before Aggregation}} \dotfill \pageref{app:case_study}\\[4pt]
\textbf{\hyperref[app:limitations]{E\quad Limitations}} \dotfill \pageref{app:limitations}\\[4pt]
\textbf{\hyperref[app:broader_impact]{F\quad Broader Impact Statement}} \dotfill \pageref{app:broader_impact}\\[4pt]
\textbf{\hyperref[app:prompt design]{G\quad Prompt Design}} \dotfill \pageref{app:prompt design}
\hypersetup{linkcolor=red}
\newpage

\section{Additional Method Details}
\label{app:method}

This appendix provides implementation details omitted from the main text, including segment construction, memory component extraction, dynamic structure maintenance, and retrieval-time uncertainty estimation.

\subsection{From Raw Messages to Segments and Components}
\label{app:construction}

\paragraph{Segment construction.}
The interaction history is first partitioned into contiguous \emph{segments}, each intended to capture a coherent local event. In practice, segmentation is performed incrementally over the message stream. A new incoming message is attached to the current segment when it remains part of the same local interaction event, and starts a new segment otherwise. This decision is based on local semantic continuity, including topical consistency, temporal continuity, and explicit discourse shifts. The resulting segments preserve intact contextual evidence that may later be needed for reconstruction.

\paragraph{Memory component extraction.}
Given a segment, \frameworkname uses an LLM to extract one or more \emph{memory components}. Each component is a concise structured unit representing a reusable fact, constraint, attribute, preference, relation, or state update expressed in the segment. The extraction prompt instructs the LLM to decompose a segment into minimal evidence units that are likely to be useful for future question answering, while avoiding unnecessary duplication across components from the same segment.

Each component is stored together with (1) its textual description, (2) an embedding for retrieval and grouping, and (3) a pointer to its source segment. This design allows the memory structure to operate over fine-grained evidence units while retaining access to the original local context.

\paragraph{Component granularity.}
We aim for components that are finer-grained than segments but still semantically self-contained. In particular, a component should isolate a single reusable piece of evidence whenever possible, rather than conflate multiple loosely related facts. This granularity is important because highly similar segments often differ only in a small but answer-critical detail. Extracting such details as separate components makes them directly retrievable at the high level.

\subsection{Detailed Structure Maintenance}
\label{app:structure}

The main text introduces a guidance objective for organising components into groups and states that the memory structure is maintained incrementally. We provide more details here.

\paragraph{Attach operation.}
Suppose a new memory component with embedding $\mathbf{x}$ arrives. Let $\boldsymbol{\mu}_k$ denote the centroid embedding of group $k$. We first identify the nearest group
\[
k^\star = \arg\max_{1 \le k \le K} \cos(\mathbf{x}, \boldsymbol{\mu}_k).
\]
If the similarity to this group exceeds a threshold $\tau_{\mathrm{attach}}$, the component is attached to group $k^\star$; otherwise a new group is created. The threshold controls whether the current organisation is considered sufficiently expressive for the incoming evidence.

\paragraph{Split triggers.}
A group is considered for splitting when it becomes too broad to serve as a useful high-level access unit. In our implementation, this happens when the group size exceeds a predefined threshold. We generate candidate partitions using graph-based local clustering on component embeddings: components are connected if their pairwise similarity exceeds a fixed threshold, and the resulting connected components form candidate clusters. A size-based fallback partition is used when necessary. Among these candidate partitions, we select the split that yields the largest improvement in Eq.~\eqref{eq:guidance}.

\paragraph{Merge triggers.}
A group is considered for merging when it contains only a single memory component and is therefore unlikely to provide a stable high-level access unit on its own. Candidate merge targets are drawn from its nearest neighbouring groups in centroid space, and the merge with the largest improvement in Eq.~\eqref{eq:guidance} is selected.

\paragraph{Maintenance schedule.}
The attach operation is applied online whenever a new component is created. Split and merge are applied periodically, or when the arrival of new components causes local structural statistics to cross the corresponding thresholds. This makes the memory structure revisable rather than fixed: later observations can reorganise earlier evidence when better relations become apparent.

\paragraph{kNN link updates.}
In addition to group membership, \frameworkname maintains top-$k$ nearest-neighbour links among groups and components in embedding space. These links are updated whenever new components are added and whenever split or merge operations alter the local structure. To avoid global recomputation after every update, kNN maintenance can be restricted to the affected neighbourhood.

\paragraph{Complexity.}
The component-to-group organisation keeps high-level retrieval efficient by reducing the number of units examined at coarse stages. Dynamic maintenance introduces additional cost, but this cost is amortised over memory updates and is confined to local structural revisions rather than full reconstruction of the memory structure. In practice, attach is inexpensive, while split and merge are triggered much less frequently.

\subsection{Detailed Retrieval Implementation}
\label{app:retrieval}

This section provides details omitted from the main text for candidate generation, backbone selection, and uncertainty-based expansion.

\paragraph{Initial candidate generation.}
Given a query $q$, we first embed the query and retrieve a candidate set of groups and memory components from the memory structure. Candidate generation is performed independently at the group level and the component level using embedding similarity. The union of these candidates forms the high-level candidate set $V$ used in Stage~I. This design allows retrieval to enter the search space both through coarse semantic regions and through directly matched fine-grained evidence units.

\paragraph{Backbone selection.}
Stage~I selects a compact set of complementary high-level units by balancing query relevance and structural coverage, as defined in Eq.~\eqref{eq:rep_select}. In practice, the procedure is greedy. At each step, the system recomputes the marginal gain of each remaining candidate under the current selected set $R$, adds the best-scoring candidate, and terminates when the candidate set is exhausted.

The coverage term is defined over local kNN neighbourhoods rather than the full graph, which keeps the computation sparse and prevents the selected set from collapsing into a single dense region of near-duplicate evidence.

\paragraph{Hierarchy-aware refinement.}
Although Eq.~\eqref{eq:rep_select} is written over a generic candidate set, we apply it hierarchically in practice. We first select a compact set of groups, then restrict attention to the components associated with those groups, and finally apply the same selection principle again to obtain the component-level evidence backbone. This reduces search cost and encourages diversity at both levels.

\paragraph{Segment expansion.}
Each selected memory component points back to its source segment. Stage~II first gathers the linked segments of the selected components and forms a coarse context consisting of the selected groups, component descriptions, and candidate segments. Segments are then considered according to their marginal uncertainty reduction. A segment is added only when it further reduces the reader's uncertainty under the current context.

\paragraph{Message expansion.}
Once a segment has been admitted, its constituent messages become eligible for finer-grained expansion. This allows the system to reveal only the most useful parts of a segment rather than always including the full span. Message-level expansion follows the same principle: a message is added only when it provides additional uncertainty reduction beyond the current context.

\paragraph{Uncertainty estimation.}
The main text uses $U(C,q)$ to denote the reader's uncertainty under context $C$ and query $q$. In our implementation, when the reader model exposes token-level predictive distributions, we instantiate $U(C,q)$ as the entropy of the next-token answer distribution under the current context. This provides a scalar proxy for how uncertain the reader remains about the answer after seeing $C$.

When the reader model does not expose token-level predictive uncertainty, we estimate $U(\cdot)$ using a proxy model that provides token-level logits. The proxy model receives the same query and context, and its predictive entropy is used only for expansion decisions; the final answer is still generated by the designated reader model. This separation allows the retrieval policy to remain uncertainty-aware even for black-box or API-based readers.

\paragraph{Stopping rule.}
Stage~II terminates when no remaining candidate segment or message yields positive marginal uncertainty reduction. This prevents unnecessary expansion into redundant raw text after the reader's uncertainty can no longer be further reduced by available candidates.

\subsection{Pseudo-code of the Full Pipeline}
\label{app:pipeline}

For completeness, we summarise the full pipeline below.

\begin{algorithm}[H]
\caption{\frameworkname}
\label{alg:xmemory}
\small
\KwIn{Interaction history $H=\{m_1,\dots,m_T\}$, query $q$, token budget $B$}
\KwOut{Retrieved context $C$}

Partition $H$ into segments\;
Extract memory components from each segment\;
Organise components into groups and build kNN links\;

\ForEach{new component arrival}{
    Attach to nearest compatible group or create a new group\;
    \If{a group becomes too large or incoherent}{
        Split using Eq.~\eqref{eq:guidance}\;
    }
    \If{a group becomes too small or isolated, such as only including one memory component}{
        Merge using Eq.~\eqref{eq:guidance}\;
    }
}

Retrieve candidate groups and components for query $q$\;
Select a compact high-level evidence backbone using Eq.~\eqref{eq:rep_select}\;

Initialise coarse context $C$ from selected groups and components\;
Gather candidate linked segments\;

\ForEach{candidate segment $s$ in ranked order}{
    \If{$\Delta U(s\mid C,q)> 0 $}{
        Add $s$ to $C$\;
        \ForEach{message $m$ in $s$}{
            \If{$\Delta U(m\mid C,q)> 0 $}{
                Add $m$ to $C$\;
            }
        }
    }
}

\Return{$C$}\;
\end{algorithm}

\section{Implementation Setup}
\label{app:implementation_details}

\paragraph{\frameworkname hyperparameters.}
Unless otherwise specified, we use a group size cap of $12$ memory components, an attachment threshold of $0.6$, a post-split clustering threshold of $0.65$, and a merge threshold of $0.65$ across all datasets and backbone models. The attachment threshold controls whether a newly created component is assigned to an existing group, while the post-split clustering threshold controls the granularity of groups formed after a split operation. The merge threshold controls when a small or weakly separated group should be merged into a neighbouring group. Additionally, following prior structured retrieval~\cite{zhang2025leanrag, sarthi2024raptor} and agent-memory systems~\cite{xu2025amem}, we set the kNN neighbourhood size to $10$ and the Stage~I candidate pool size to $20$. Experiments with open-source models are run on an NVIDIA A100 80GB GPU.

\paragraph{Hyperparameter sensitivity.}
We evaluate the sensitivity of \frameworkname to the main structure-update thresholds on LoCoMo using Qwen3-8B for both memory construction and answer generation. For each experiment, we vary one threshold while keeping all other hyperparameters fixed. As shown in Table~\ref{tab:threshold_sensitivity}, the default attachment threshold of $0.6$, post-split clustering threshold of $0.65$, and merge threshold of $0.65$ achieve the best average BLEU and F1 in their respective sensitivity groups. The method remains relatively stable across nearby values, suggesting that its performance does not depend on a narrow threshold choice. Lower attachment thresholds may attach components too aggressively, while higher thresholds make the structure more conservative. Similarly, lower post-split thresholds may produce overly coarse groups, whereas higher thresholds may over-fragment related components. For merging, both lower and higher thresholds lead to weaker average performance, indicating that an intermediate value better preserves the balance between structural compactness and semantic separation.

\begin{table*}[t]
\centering
\small
\setlength{\tabcolsep}{3.5pt}
\renewcommand{\arraystretch}{1.08}
\caption{Sensitivity analysis of the main structure-update thresholds on LoCoMo using Qwen3-8B. We vary one threshold at a time while keeping all other hyperparameters fixed. Results are reported as BLEU/F1, and the default values are marked in gray.}
\label{tab:threshold_sensitivity}
\resizebox{\textwidth}{!}{%
\begin{tabular}{llccccc}
\toprule
\textbf{Threshold} & \textbf{Value} 
& \textbf{Multi-hop} 
& \textbf{Temporal} 
& \textbf{Open-domain} 
& \textbf{Single-hop} 
& \textbf{Average} \\
\midrule

\multirow{4}{*}{Attachment}
& $0.5$ & 25.36 / 37.29 & 28.13 / 35.38 & 14.97 / 19.12 & 39.86 / 49.65 & 33.21 / 42.51 \\
& \cellcolor{gray!12}$0.6$ & \cellcolor{gray!12}27.24 / 38.57 & \cellcolor{gray!12}29.58 / 37.46 & \cellcolor{gray!12}15.55 / 20.69 & \cellcolor{gray!12}40.94 / 50.94 & \cellcolor{gray!12}34.48 / 43.98 \\
& $0.7$ & 25.69 / 38.02 & 27.56 / 35.71 & 15.27 / 20.02 & 40.09 / 49.81 & 33.29 / 42.85 \\
& $0.8$ & 24.95 / 36.66 & 27.35 / 35.65 & 14.20 / 18.43 & 39.56 / 49.42 & 32.76 / 42.28 \\
\midrule

\multirow{3}{*}{Post-split clustering}
& $0.55$ & 26.82 / 37.91 & 28.82 / 36.78 & 14.59 / 20.06 & 39.10 / 49.20 & 33.18 / 42.73 \\
& \cellcolor{gray!12}$0.65$ & \cellcolor{gray!12}27.24 / 38.57 & \cellcolor{gray!12}29.58 / 37.46 & \cellcolor{gray!12}15.55 / 20.69 & \cellcolor{gray!12}40.94 / 50.94 & \cellcolor{gray!12}34.48 / 43.98 \\
& $0.75$ & 26.80 / 38.22 & 28.85 / 36.27 & 14.06 / 19.79 & 39.86 / 50.51 & 33.57 / 43.37 \\
\midrule

\multirow{4}{*}{Merge}
& $0.6$ & 26.34 / 37.21 & 28.65 / 36.51 & 14.88 / 19.40 & 39.70 / 49.21 & 33.40 / 42.51 \\
& \cellcolor{gray!12}$0.65$ & \cellcolor{gray!12}27.24 / 38.57 & \cellcolor{gray!12}29.58 / 37.46 & \cellcolor{gray!12}15.55 / 20.69 & \cellcolor{gray!12}40.94 / 50.94 & \cellcolor{gray!12}34.48 / 43.98 \\
& $0.7$ & 26.05 / 37.10 & 28.23 / 35.78 & 15.19 / 19.98 & 39.34 / 48.80 & 33.08 / 42.14 \\
& $0.75$ & 25.43 / 36.16 & 27.83 / 35.37 & 14.59 / 19.68 & 39.19 / 48.21 & 32.77 / 41.55 \\

\bottomrule
\end{tabular}
}
\end{table*}

\paragraph{Existing assets and licenses.} We use existing datasets, models, and baseline methods only for research evaluation. LoCoMo and PerLTQA are credited through their original papers in Section~4.1. PerLTQA is released under the CC BY-NC 4.0 license for non-commercial research use. LoCoMo is released with its accompanying data and code repository; we follow the terms provided by the dataset authors. We use Qwen3-8B, Llama-3.1-8B-Instruct, GPT-5 nano, GPT-4.1-mini, and \texttt{text-embedding-3-small} according to their respective model licenses or API terms. Baseline methods are credited through their original papers, and any reused code in the released package follows the corresponding licenses and attribution requirements.

\section{More Experiment Results}

\subsection{Detailed Ablation Results}
\label{app:ablation_details}

Table~\ref{tab:locomo_ablation_full} reports the full category-level ablation results for the retrieval design on LoCoMo with Qwen3-8B. 
The results confirm that memory structuring alone already provides a substantially stronger retrieval basis than Naive RAG, and that both Stage~I and Stage~II contribute further improvements from complementary directions. 
Stage~I mainly improves high-level evidence selection while reducing token usage, whereas Stage~II contributes stronger gains in answer quality by uncertainty-guided expansion over segments and messages. 
The full system achieves the best overall trade-off across performance and inference-time token cost.

\begin{table*}[t]
\centering
\small
\setlength{\tabcolsep}{4.2pt}
\renewcommand{\arraystretch}{1.12}
\caption{\textbf{Ablation on LoCoMo with Qwen3-8B.}
We report BLEU, F1, and Token/query (lower is better).
\textbf{Memory-only} uses the full memory structure but replaces both retrieval stages with basic similarity retrieval.
\textbf{+Stage I} adds group and component selection.
\textbf{+Stage II} adds uncertainty-based segment and message expansion.
\textbf{Ours (Full)} combines both stages.}
\label{tab:locomo_ablation_full}
\resizebox{\linewidth}{!}{
\begin{tabular}{lrrrrrrrrrrr}
\toprule
\textbf{Setting} &
\multicolumn{2}{c}{\textbf{Multi-hop}} &
\multicolumn{2}{c}{\textbf{Temporal}} &
\multicolumn{2}{c}{\textbf{Open-domain}} &
\multicolumn{2}{c}{\textbf{Single-hop}} &
\multicolumn{2}{c}{\textbf{Average}} &
\multicolumn{1}{c}{\textbf{Token}} \\
\cmidrule(lr){2-3}\cmidrule(lr){4-5}\cmidrule(lr){6-7}\cmidrule(lr){8-9}\cmidrule(lr){10-11}
& \textbf{BLEU} & \textbf{F1}
& \textbf{BLEU} & \textbf{F1}
& \textbf{BLEU} & \textbf{F1}
& \textbf{BLEU} & \textbf{F1}
& \textbf{BLEU} & \textbf{F1}
& \textbf{/query} \\
\midrule
Naive RAG             & 22.46 & 34.28 & 17.21 & 21.32 & 12.35 & 17.08 & 35.68 & 45.22 & 27.95 & 36.48 & 8633.28 \\
Memory-only           & 24.48 & 35.41 & 27.28 & 34.14 & 13.20 & 17.75 & 38.19 & 47.73 & 31.81 & 40.77 & 7235.56 \\
+Stage I              & 26.55 & 38.02 & 28.08 & 35.72 & 14.30 & 19.75 & 38.36 & 45.79 & 32.55 & 40.64 & 6320.72 \\
+Stage II             & 26.64 & 37.03 & 28.92 & 37.09 & 14.10 & 20.43 & 39.28 & 49.01 & 33.23 & 42.55 & 6556.58 \\
\frameworkname (Full) & \textbf{27.24} & \textbf{38.57} & \textbf{29.58} & \textbf{37.46} & \textbf{15.55} & \textbf{20.69} & \textbf{40.94} & \textbf{50.94} & \textbf{34.48} & \textbf{43.98} & \textbf{4711.29} \\
\bottomrule
\end{tabular}
}
\end{table*}

\subsection{Theoretical Motivation of Branching Factor in Hierarchical Dialogue Summarization}
\label{app:fano_motivation}

\subsubsection{Motivation and Fano Inequality}

Our hierarchy performs top-down retrieval, in which key subroutines repeatedly \emph{route} to a relevant unit by selecting from a candidate set (e.g., a group/memory component relevant to the current query context). In high-similarity dialogue streams, the observable signal used for routing (summaries/embeddings/LLM scores) can be inherently weak at discriminating among many near-duplicate candidates. This subsection summarises a standard information-theoretic implication: with bounded discriminative information, the error of a multi-way routing decision cannot be made arbitrarily small as the candidate set grows.

Consider one such routing decision that must identify the correct option among $n_k$ candidates within a group on the component--group level. Let $Z\in\{1,\dots,n_k\}$ be the (unknown) index of the correct candidate, and let $O$ denote the observable evidence (e.g., the group summary derived from its constituent components and the query/context representations) used for routing. Any routing rule outputs $\hat Z(O)$ with error probability $p_e = \Pr[\hat Z \neq Z]$. A classical result (Fano's inequality) implies

\begin{equation}
p_e \;\ge\; 1 - \frac{I(Z;O)+1}{\log_2 n_k},
\label{eq:fano}
\end{equation}

where $I(Z;O)$ denotes the mutual information between the routing evidence $O$ (derived from the query/context and stored representations) and the correct candidate index $Z$. In our setting, candidates within the same group are often near-duplicates, and the evidence is produced from compressed LLM representations (with additional noise from imperfect assignments). As a result, the \emph{discriminative} information available for distinguishing among candidates is bounded, i.e., $I(Z;O)$ is small. By Fano's inequality, a bounded $I(Z;O)$ implies a non-trivial lower bound on the routing error that increases with $\log n_k$. Therefore, to achieve a low misrouting rate (small $p_e$), it is necessary to control the candidate set size $n_k$. This aligns with the practical intuition: when many candidates are highly similar, identifying the source becomes close to guessing unless we reduce the arity of the routing problem.

\subsubsection{Connection to Metric Designing}

Equation~\eqref{eq:fano} shows that if $I(Z;O)$ is bounded (as expected when many candidates are highly similar), then increasing $n_k$ necessarily increases a lower bound on the misrouting probability. 

In our group partition $P=\{C_k\}_{k=1}^K$ of $N$ component nodes with $n_k = |C_k|$, the typical within-group candidate set size directly governs the arity of routing within a group. Eq.~(2) estimates the expected within-group scanning cost as
$E[n_{\mathrm{cluster}}] = \frac{1}{N}\sum_{k=1}^K n_k^2$,
and defines
$\mathrm{SparsityScore}(P) = \frac{N^2}{K\sum_{k=1}^K n_k^2}$.
Maximising $\mathrm{SparsityScore}(P)$ thus controls the typical within-group candidate size (and prevents extremely large groups), which is precisely the regime where the lower bound in Eq.~\eqref{eq:fano} becomes prohibitive. This provides a theoretical motivation for balancing group sizes: it limits unavoidable routing errors under bounded discriminative evidence in high-similarity dialogue memory.

\begin{table*}[t]
\centering
\small
\setlength{\tabcolsep}{3.9pt}
\renewcommand{\arraystretch}{1.12}
\caption{\textbf{Effect of the group size cap on LoCoMo (Qwen3-8B).}
We vary the upper bound of memory component nodes per group and report BLEU/F1 by question category and the average, together with resulting structure statistics.}
\label{tab:theme_cap_sweep}
\resizebox{\textwidth}{!}{
\begin{tabular}{lrrrrrrrrrrrrr}
\toprule
\textbf{Upper bound} &
\multicolumn{2}{c}{\textbf{Multi-hop}} &
\multicolumn{2}{c}{\textbf{Temporal}} &
\multicolumn{2}{c}{\textbf{Open-domain}} &
\multicolumn{2}{c}{\textbf{Single-hop}} &
\multicolumn{2}{c}{\textbf{Average}} &
\multicolumn{1}{c}{\makecell{\textbf{\#}\\\textbf{Com.}}} &
\multicolumn{1}{c}{\makecell{\textbf{\#}\\\textbf{Groups}}} &
\multicolumn{1}{c}{\makecell{\textbf{Avg. com.}\\\textbf{per group}}} \\
\cmidrule(lr){2-3}\cmidrule(lr){4-5}\cmidrule(lr){6-7}\cmidrule(lr){8-9}\cmidrule(lr){10-11}
& \textbf{BLEU} & \textbf{F1}
& \textbf{BLEU} & \textbf{F1}
& \textbf{BLEU} & \textbf{F1}
& \textbf{BLEU} & \textbf{F1}
& \textbf{BLEU} & \textbf{F1}
& & & \\
\midrule
14 & 26.31 & 37.44 & 28.67 & 36.33 & 13.78 &18.41 & 39.43 & 48.57 & 33.19 & 42.10	 &2861  & 383  & 7.47 \\
12 & \textbf{27.24} & \textbf{38.57} & \textbf{29.58} & \textbf{37.46} & \textbf{15.55} & \textbf{20.69} & \textbf{40.94} & \textbf{50.94} & \textbf{34.48} & \textbf{43.98} & 2879 & 642  & 4.48 \\
10 & 26.52	 & 37.62 & 28.93 & 36.78 & 14.12 & 18.02 & 39.31 & 49.66 & 33.23 & 42.80 & 2932 & 805  & 3.64 \\
8  & 25.01	& 36.31 & 27.46 &35.24 & 12.91 & 17.82 &38.49  & 47.81 & 32.13 & 41.21 & 2862 & 1111 &  2.58\\
\bottomrule
\end{tabular}
}
\end{table*}

\subsubsection{Fano-style Lower Bound and optimal $n_k$ in hierarchical structure.}
\label{app:k-proof}

In this subsection, we use the Fano-style lower bound to guide us to find the optimal $n_k$.

\begin{theorem}
Let $Z \in \{1,\dots,n_k\}$ denote the (unknown) correct index and let $O$ be the observable evidence used to infer $Z$.
For any estimator $\hat Z = g(O)$ with error probability $p_e = \Pr[\hat Z \neq Z]$,
\begin{equation}
H(Z\mid O) \le h(p_e) + p_e \log_2(n_k-1),
\label{eq:fano_standard}
\end{equation}
where $h(\cdot)$ is the binary entropy function.
In particular, if $Z$ is uniform on $\{1,\dots,k\}$, then
\begin{equation}
p_e \ge 1 - \frac{I(Z;O)+1}{\log_2 n_k}.
\label{eq:fano_simplified}
\end{equation}
\end{theorem}

\begin{proof}
Equation~\eqref{eq:fano_standard} is the standard form of Fano's inequality.
Using $I(Z;O) = H(Z) - H(Z\mid O)$ and rearranging yields
\[
H(Z) - I(Z;O) \le h(p_e) + p_e \log_2(n_k-1).
\]
Since $h(p_e) \le 1$ and $\log_2(n_k-1) \le \log_2 n_k$, we have
\[
H(Z) - I(Z;O) \le 1 + p_e \log_2 n_k.
\]
If $Z$ is uniform, then $H(Z)=\log_2 n_k$, hence
\[
\log_2 n_k - I(Z;O) \le 1 + p_e \log_2 n_k
\quad\Rightarrow\quad
p_e \ge 1 - \frac{I(Z;O)+1}{\log_2 n_k},
\]
which is~\eqref{eq:fano_simplified}.
\end{proof}

\begin{corollary}[Admissible candidate set size under bounded discriminability]
\label{cor:k_cap}
Assume the routing evidence has bounded discriminative information $I(Z;O)\le B$ (in bits).
If we require $p_e \le \varepsilon$ for some $\varepsilon \in (0,1)$, then any feasible $n_k$ must satisfy
\begin{equation}
\log_2 n_k \le \frac{B+1}{1-\varepsilon}
\quad\Longrightarrow\quad
n_k \le 2^{\frac{B+1}{1-\varepsilon}}.
\label{eq:k_cap}
\end{equation}
\end{corollary}

\begin{proof}
Rearrange~\eqref{eq:fano_simplified} and substitute $I(Z;O)\le B$.
\end{proof}

To instantiate Corollary~\ref{cor:k_cap}, we need a conservative upper bound $B$ on the mutual information $I(Z;O)$ between the routing evidence and the correct candidate index. Within a group of near-duplicate components, the summary signal $O$ carries only a few bits of discriminative information beyond what the components share; we therefore adopt $B=2$ bits, an upper-bounding assumption that lets an oracle perfectly distinguish at most $2^B=4$ candidates without further side information. With a routing accuracy target of $\alpha = 1-\varepsilon = 0.85$, Corollary~\ref{cor:k_cap} yields
\begin{equation}
n_k \;\le\; 2^{(B+1)/\alpha} \;=\; 2^{3.529} \;\approx\; 11.5, \nonumber
\end{equation}
which we round to $n_k^{\max}=12$ as the split threshold in our implementation.

This information-theoretic prediction is corroborated empirically. Table~\ref{tab:theme_cap_sweep} sweeps the cap from $8$ to $14$ on LoCoMo with Qwen3-8B and shows that $n_k^{\max}=12$ achieves the best BLEU/F1 in \emph{every} question category and the best average ($34.48/43.98$). Caps above $12$ (e.g., $14$) reduce performance because larger groups exceed the routing-discriminability budget predicted by Fano, while caps below $12$ (e.g., $10$ or $8$) over-fragment the hierarchy and weaken multi-fact coverage.

The actual average branching factor under the cap of $12$ is $\approx 4.5$ components per group. The cap is a ceiling rather than a target: the SparsityScore + SemScore objective in Eq.~\eqref{eq:guidance} produces a distribution of group sizes well concentrated below $n_k^{\max}$, with only the largest groups approaching the Fano-admissible limit. The two mechanisms are complementary: Fano supplies a worst-case ceiling that prevents pathologically large groups, while the objective handles fine-grained sizing within the safe regime.

\subsection{Full Performance in Evidence Density under RAG and Pruning Effect}
\label{app:rag effect}

Table~\ref{tab:app_evidence_density_coverage} provides the detailed performance and coverage statistics for the three settings on LoCoMo with Qwen3-8B.
RAG with pruning shows mixed behaviour across categories: while it slightly improves BLEU on multi-hop and open-domain questions, it reduces F1 and consistently underperforms Naive RAG on single-hop and the overall average.
This is consistent with Figure~\ref{fig:hit-ratio-by-category}, where pruning shifts retrieved blocks from 2-hit and multi-hit toward 1-hit, suggesting that compression often discards answer-bearing details.
In contrast, \frameworkname improves both BLEU and F1 in every category, with particularly large gains on temporal and single-hop questions, and achieves the highest overall average.
It also covers answer evidence more efficiently, requiring fewer blocks for full coverage (5.66 vs.\ 10.81 and 13.31) and substantially lower coverage token cost (974.56 vs.\ 1979.26 and 1587.99), indicating more answer-sufficient retrieval under a tighter budget.

\begin{table*}[t]
\centering
\small
\setlength{\tabcolsep}{3.9pt}
\renewcommand{\arraystretch}{1.12}
\caption{\textbf{Expanded results for evidence density and coverage efficiency on LoCoMo (Qwen3-8B).}
We report BLEU and F1 for each question category and the average. \emph{Avg. blocks for coverage} is the average number of retrieved blocks required to cover all answer content tokens, and \emph{Cover token cost/query} is the corresponding token cost (lower is better).}
\label{tab:app_evidence_density_coverage}
\resizebox{\textwidth}{!}{
\begin{tabular}{lrrrrrrrrrrrcc}
\toprule
\textbf{Framework} &
\multicolumn{2}{c}{\textbf{Multi-hop}} &
\multicolumn{2}{c}{\textbf{Temporal}} &
\multicolumn{2}{c}{\textbf{Open-domain}} &
\multicolumn{2}{c}{\textbf{Single-hop}} &
\multicolumn{2}{c}{\textbf{Average}} &
\multicolumn{1}{c}{\makecell{\textbf{Avg. blocks}$\downarrow$\\\textbf{for coverage}}} &
\multicolumn{1}{c}{\makecell{\textbf{Avg. tokens$\downarrow$}\\\textbf{for coverage}}} \\
\cmidrule(lr){2-3}\cmidrule(lr){4-5}\cmidrule(lr){6-7}\cmidrule(lr){8-9}\cmidrule(lr){10-11}
& \textbf{BLEU} & \textbf{F1}
& \textbf{BLEU} & \textbf{F1}
& \textbf{BLEU} & \textbf{F1}
& \textbf{BLEU} & \textbf{F1}
& \textbf{BLEU} & \textbf{F1}
& & \\
\midrule
Naive RAG     & 22.46 & 34.28 & 17.21 & 21.32 & 12.35 & 17.08 & 35.68 & 45.22 & 27.95 & 36.48 & 10.81 & 1979.26 \\
RAG + Pruning & 23.81 & 33.58 & 16.53 & 23.33 & 12.92 & 16.22 & 32.85 & 41.23 & 26.55 & 34.58 & 13.31 & 1587.99 \\
\frameworkname (Ours) & \textbf{27.24} & \textbf{38.57} & \textbf{29.58} & \textbf{37.46} & \textbf{15.55} & \textbf{20.69} & \textbf{40.94} & \textbf{50.94} & \textbf{34.48} & \textbf{43.98} & \textbf{5.66} & \textbf{974.56} \\
\bottomrule
\end{tabular}
}
\end{table*}

\subsection{Full Performance of Retroactive Restructuring in Memory Construction}
\label{app:reassign_details}

\paragraph{Extended results.}
Table~\ref{tab:reassign_full_bycat} provides category-level QA results together with the dynamic reassignment ratio for different construction settings.
Table~\ref{tab:reassign_structure_stats} reports the resulting numbers of group and memory component nodes.

\begin{table*}[t]
\centering
\small
\setlength{\tabcolsep}{3.9pt}
\renewcommand{\arraystretch}{1.12}
\caption{\textbf{Extended results for retroactive restructuring on LoCoMo (Qwen3-8B).}
We report BLEU/F1 by question category and the average, together with the dynamic reassignment ratio under each construction setting.}
\label{tab:reassign_full_bycat}
\resizebox{\textwidth}{!}{
\begin{tabular}{lrrrrrrrrrrrc}
\toprule
\textbf{Setting} &
\multicolumn{2}{c}{\textbf{Multi-hop}} &
\multicolumn{2}{c}{\textbf{Temporal}} &
\multicolumn{2}{c}{\textbf{Open-domain}} &
\multicolumn{2}{c}{\textbf{Single-hop}} &
\multicolumn{2}{c}{\textbf{Average}} &
\multicolumn{1}{c}{\makecell{\textbf{Reassign}\\\textbf{ratio}}} \\
\cmidrule(lr){2-3}\cmidrule(lr){4-5}\cmidrule(lr){6-7}\cmidrule(lr){8-9}\cmidrule(lr){10-11}
& \textbf{BLEU} & \textbf{F1}
& \textbf{BLEU} & \textbf{F1}
& \textbf{BLEU} & \textbf{F1}
& \textbf{BLEU} & \textbf{F1}
& \textbf{BLEU} & \textbf{F1}
& \\
\midrule
w/o merge\&split & 24.01 & 34.35 & 28.24 & 34.90 & 10.97 & 14.04 & 37.24 & 44.23 & 31.30 & 38.59 & 0.00\% \\
w/o split        & 24.46 & 35.46 & 28.55 & 34.77 & 11.61 & 15.75 & 37.69 & 44.79 & 31.73 & 39.18 & 17.33\% \\
w/o merge        & 24.81 & 35.46 & 29.68 & 35.99 & 13.09 & 18.09 & 38.45 & 46.16 & 32.54 & 40.33 & 26.23\% \\
\frameworkname (Full) & \textbf{27.24} & \textbf{38.57} & \textbf{29.58} & \textbf{37.46} & \textbf{15.55} & \textbf{20.69} & \textbf{40.94} & \textbf{50.94} & \textbf{34.48} & \textbf{43.98} & \textbf{44.91}\% \\
\bottomrule
\end{tabular}
}
\end{table*}

\begin{table}[t]
\centering
\small
\setlength{\tabcolsep}{5.4pt}
\renewcommand{\arraystretch}{1.12}
\caption{\textbf{Structure statistics under different construction settings.}
We report the resulting numbers of group and memory component nodes.}
\label{tab:reassign_structure_stats}
\begin{tabular}{lcc}
\toprule
Setting & \#Groups & \#Components \\
\midrule
w/o merge\&split & 1114 & 2876 \\
w/o split        & 651  & 2897 \\
w/o merge        & 1122 & 2928 \\
Full             & 642  & 2879 \\
\bottomrule
\end{tabular}
\end{table}

\paragraph{Summary.}
Disabling both operators yields a static hierarchy with 0\% reassignment and the lowest overall accuracy.
Allowing split or merge enables retroactive reassignment and improves QA, while the full system achieves the highest reassignment ratio and the best accuracy across all categories.
Split accounts for a larger share of reassignment, whereas merge is important for consolidating redundant groups and producing a compact high-level organisation, as reflected by the reduced number of groups under the full setting.

\subsection{Additional Cost--Performance Statistics}
\label{app:cost_perf_stats}

Table~\ref{tab:app_cost_perf_stats} reports the token and performance statistics used in Figure~\ref{fig:locomo_cost_tradeoff}. Following the LightMem-style reporting setup~\cite{fang2026lightmem}, We report the average construction cost and inference-time token cost per query, and average BLEU/F1 on LoCoMo with Qwen3-8B. Methods marked with \flatmark{} are flat retrieval baselines, and methods marked with \structmark{} are structured memory baselines. The table shows that \frameworkname introduces moderate construction overhead compared with Nemori, but substantially reduces inference-time token cost while achieving the highest BLEU and F1.

\begin{table}[t]
\centering
\small
\setlength{\tabcolsep}{6pt}
\renewcommand{\arraystretch}{1.12}
\caption{
\textbf{Token and performance statistics for the cost--performance analysis on LoCoMo with Qwen3-8B.}
Construction Token denotes the memory construction cost in thousands of tokens, following the LightMem-style reporting setup.
Token/query denotes the average inference-time token cost per query.
Methods marked with \flatmark{} are flat retrieval baselines, and methods marked with \structmark{} are structured memory baselines.
}
\label{tab:app_cost_perf_stats}
\resizebox{0.88\textwidth}{!}{%
\begin{tabular}{lrrrr}
\toprule
\textbf{Method} &
\makecell{\textbf{Construction}\\\textbf{Token (k)}} &
\makecell{\textbf{Token}\\\textbf{/query}} &
\textbf{Avg BLEU} &
\textbf{Avg F1} \\
\midrule
Full Memory\flatmark & 0.00 & 18535.90 & 25.83 & 33.54 \\
Naive RAG\flatmark   & 0.00 & 8633.28  & 27.95 & 36.48 \\
LightMem\flatmark    & 125.95 & 5545.35 & 26.04 & 33.66 \\
\cmidrule(lr){1-5}
Nemori\structmark    & 381.29 & 7754.66 & 31.44 & 40.88 \\
A-Mem\structmark     & 1008.55 & 9103.46 & 27.84 & 37.13 \\
MemoryOS\structmark  & 481.89 & 7234.66 & 29.20 & 33.76 \\
\cmidrule(lr){1-5}
\frameworkname\textbf{ (Ours)} & 421.91 & \textbf{4711.29} & \textbf{34.48} & \textbf{43.98} \\
\bottomrule
\end{tabular}%
}
\end{table}

\section{Case Study: Decoupling before Aggregation}
\label{app:case_study}

To illustrate the memory constructed by \frameworkname, Figure~\ref{fig:case_study_gina} shows a qualitative example from LoCoMo. The query asks: \emph{``When did Gina lose her job at DoorDash?''}, whose gold answer is \emph{``January 2023.''} The original memory stream contains several related events: Gina mentions losing her DoorDash job in January, later refers again to the job loss, and subsequently discusses opening and promoting an online clothing store. The challenge is therefore not that relevant information is missing, but that the decisive temporal evidence is embedded within a cluster of highly similar memories.

The left side of Figure~\ref{fig:case_study_gina} shows how flat similarity-based retrieval behaves over this raw memory stream. Since later memories are semantically close to the query, top-$k$ retrieval selects broadly relevant memories about Gina's job loss and business transition. However, these memories mix the original job-loss event with its later consequences, so the temporal anchor is not clearly isolated and the reader produces a vague answer.

\begin{figure*}[t]
    \centering
    \includegraphics[width=\linewidth]{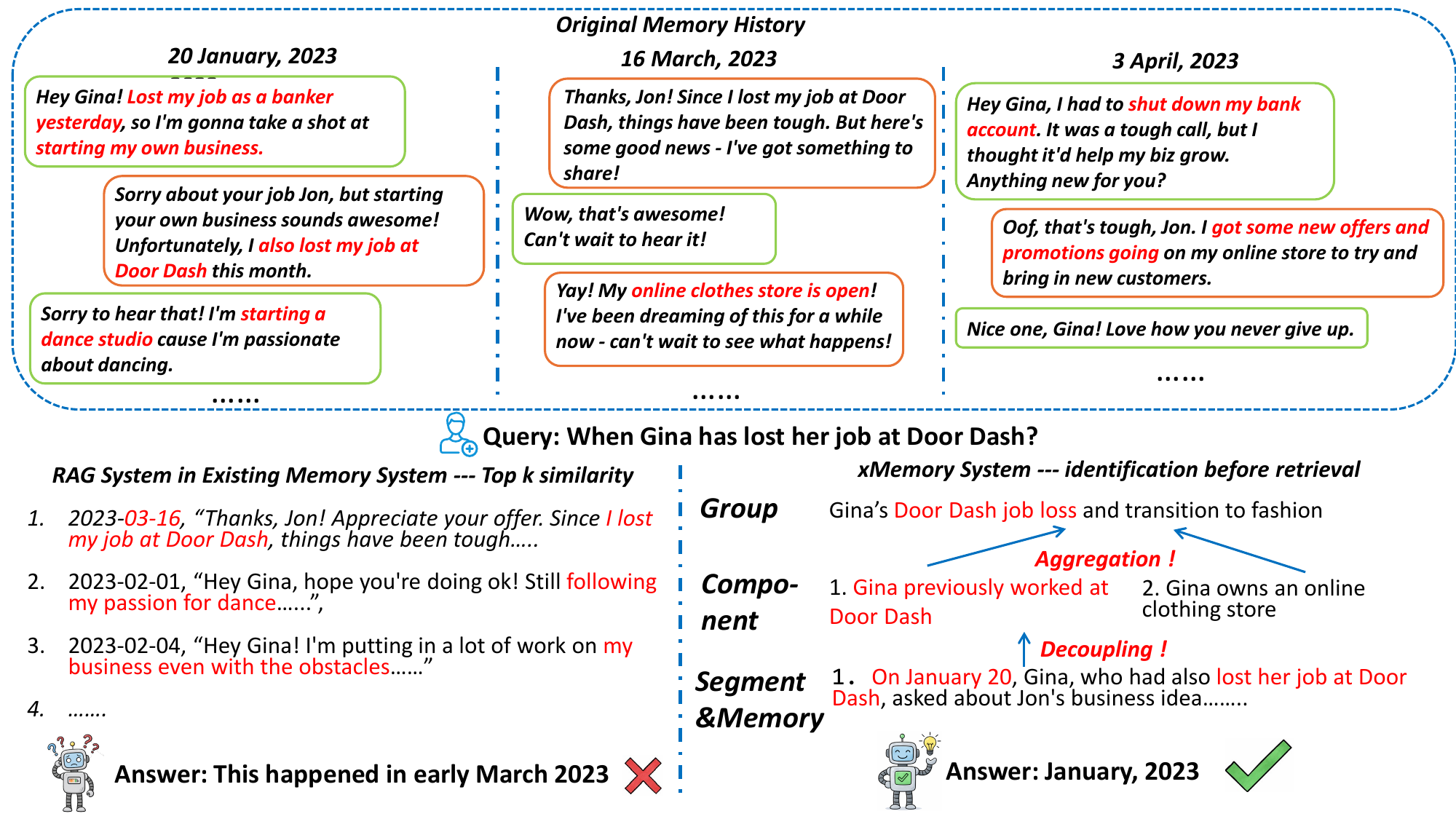}
    \caption{\textbf{Case study on LoCoMo.}
    The query asks \emph{``When did Gina lose her job at DoorDash?''} The original memory stream contains several similar events about Gina's DoorDash job loss and later clothing-store transition. Flat similarity retrieval selects broadly relevant memories from this cluster, but mixes the original job-loss event with later consequences. In contrast, \frameworkname organises the same history into a hierarchy: a group provides coarse access to the relevant memory region, components separate the DoorDash job-loss fact from related business-transition facts, and the linked source segment preserves the January evidence needed for answer generation.}
    \label{fig:case_study_gina}
\end{figure*}

The right side of Figure~\ref{fig:case_study_gina} illustrates how \frameworkname represents the same evidence through its hierarchy. The group provides compact access to Gina's employment and business-transition memories. Within this group, separate memory components distinguish related but different facts, such as Gina's prior DoorDash job and her later online clothing store. The selected component then points back to the source segment, where the January interaction provides the original local context and temporal evidence.

This example illustrates why the hierarchy is not merely a compressed summary of the raw history. Groups provide coarse access to related evidence regions, components preserve distinctions among similar events, and source-linked segments retain the original context needed for final answer generation.

\section{Limitations}
\label{app:limitations}

This work focuses on retrieval and memory organisation for long-term agent memory, and its empirical evaluation is conducted on LoCoMo and PerLTQA. Although these benchmarks cover multi-session dialogue and personalised long-term memory, they do not exhaust all possible agent-memory settings. We therefore view extending the evaluation to broader agentic settings as an important direction for future work.

xMemory also introduces additional memory-construction cost because it extracts memory components and maintains a revisable hierarchy. Our experiments include an end-to-end cost--performance analysis that accounts for construction and inference tokens, but the practical trade-off may depend on the frequency of memory updates and future queries. In applications with very few queries after memory construction, simpler retrieval methods may be preferable.

Finally, the current implementation relies on LLM-generated memory components and summaries. While the structure is designed to reduce redundancy and preserve links to source segments, errors in component extraction may still affect downstream retrieval. Future work could investigate more robust extraction, verification, and privacy-preserving memory maintenance.

\section{Broader Impact Statement}
\label{app:broader_impact}

This work advances agent-memory retrieval for long-horizon interactions by reducing redundant context while preserving temporally linked evidence. It may improve the reliability and efficiency of applications such as personal assistants, long-term dialogue systems, and multi-session decision-support tools, while reducing inference cost through lower token usage.

Potential risks include privacy leakage, unintended retention of sensitive user information, and harmful reliance on outdated or incorrect memories when deployed with real user data. More effective memory retrieval may also increase the importance of responsible memory management, since retrieved information can influence downstream agent behaviour. Our work does not introduce new data collection and is evaluated on existing benchmarks, but practical deployments should incorporate user consent, data minimisation, access control, retention and deletion policies, and mechanisms for users to inspect, correct, or override stored memories. Overall, we do not anticipate societal impacts beyond those commonly associated with retrieval-augmented LLM systems, but responsible use requires careful privacy and safety safeguards.


\section{Prompt Design}
\label{app:prompt design}

In this section, we present the specific prompts used for the experiments in Section \ref{sec:implementation_details}.

\subsection{Memory Structure Management}
This appendix lists the prompts used to construct our hierarchical memory from fine to coarse granularity.
We first apply a boundary detector to segment the dialogue stream into coherent segments.
For each segment, we generate an episodic memory record with a title, third-person narrative content, and an explicit timestamp.
We then distill high-value, persistent memory components from accumulated segments to form reusable long-term knowledge, filtering out transient conversation details.
Finally, we summarise clusters of related memory component statements into concise group descriptions that serve as stable high-level indices for structure management and retrieval.

\begin{table*}[!htbp]

    \centering
    \small
    \begin{tcolorbox}[width=\textwidth, title={\textbf{Prompt for Boundary Detect}}]

You are a dialogue boundary detection expert. You need to determine if the newly added dialogue should end the current episode and start a new one as a segment.

Current conversation history:
$\{conversation\ history\}$

Newly added messages:
$\{new\ messages\}$

Please carefully analyze the following aspects to determine if a new segment should begin:

1. **Topic Change** (Highest Priority):
   - Do the new messages introduce a completely different topic?
   - Is there a shift from one specific event to another?
   - Has the conversation moved from one question to an unrelated new question?

2. **Intent Transition**:
   - Has the purpose of the conversation changed? (e.g., from casual chat to seeking help, from discussing work to discussing personal life)
   - Has the core question or issue of the current topic been answered or fully discussed?

3. **Temporal Markers**:
   - Are there temporal transition markers ("earlier", "before", "by the way", "oh right", "also", etc.)?
   - Is the time gap between messages more than 30 minutes?

4. **Structural Signals**:
   - Are there explicit topic transition phrases ("changing topics", "speaking of which", "quick question", etc.)?
   - Are there concluding statements indicating the current topic is finished?

5. **Content Relevance**:
   - How related is the new message to the previous discussion? (Consider splitting if relevance < 30
   - Does it involve completely different people, places, or events?

Decision Principles:
- **Prioritize topic independence**: Each episodic segment should revolve around one core topic or event
- **When in doubt, split**: When uncertain, lean towards starting a new episodic segment
- **Maintain reasonable length**: A single episodic segment typically shouldn't exceed 10-15 messages

Note:
- If conversation history is empty, this is the first message, return false
- When a clear topic change is detected, split even if the conversation flows naturally
- Each episodic segment should be a self-contained conversational unit that can be understood independently

    \end{tcolorbox}
        \label{tab:prompt for boundary detect}
\end{table*}

\begin{table*}[!htbp]

    \centering
    \small
    \begin{tcolorbox}[width=\textwidth, title={\textbf{Prompt for Episodic Generation}}]

You are an episodic segment memory generation expert. Please convert the following conversation into an episodic segment memory.

Conversation content:
${conversation}$

Boundary detection reason:
${boundary\ reason}$

Please analyze the conversation to extract time information and generate a structured episodic memory. Return containing the following three fields:
{{
    "title": "A concise, descriptive title that accurately summarises the group (10-20 words)",
    "content": "A detailed description of the conversation in third-person narrative. It must include all important information: who participated in the conversation at what time, what was discussed, what decisions were made, what emotions were expressed, and what plans or outcomes were formed. Write it as a coherent story so that the reader can clearly understand what happened. Ensure that time information is precise to the hour, including year, month, day, and hour.",
    "timestamp": "YYYY-MM-DDTHH:MM:SS format timestamp representing when this episodic segment occurred (analyse from message timestamps or content)"
}}

Time Analysis Instructions:
1. **Primary Source**: Look for explicit timestamps in the message metadata or content
2. **Secondary Source**: Analyze temporal references in the conversation content ("yesterday", "last week", "this morning", etc.)
3. **Fallback**: If no time information is available, use a reasonable estimate based on context
4. **Format**: Always return timestamp in ISO format: "2024-01-15T14:30:00"

Requirements:
1. The title should be specific and easy to search (including key topics/activities).
2. The content must include all important information from the conversation.
3. Convert the dialogue format into a narrative description.
4. Maintain chronological order and causal relationships.
5. Use third-person unless explicitly first-person.
6. Include specific details that aid keyword search.
7. Notice the time information, and write the time information in the content.
8. When relative times (e.g., last week, next month, etc.) are mentioned in the conversation, you need to convert them to absolute dates (year, month, day). Write the converted time in parentheses after the original time reference.
9. **IMPORTANT**: Analyze the actual time when the conversation happened from the message timestamps or content, not the current time.

Example:
If the conversation is about someone planning to go hiking and the messages have timestamps from March 14, 2024 at 3:00 PM:
{{
    "title": "Weekend Hiking Plan March 16, 2024: Sunrise Trip to Mount Rainier",
    "content": "On March 14, 2024 at 3:00 PM, the user expressed interest in going hiking on the upcoming weekend (March 16, 2024) and sought advice. They particularly wanted to see the sunrise at Mount Rainier, having heard the scenery is beautiful. When asked about gear, they received suggestions including hiking boots, warm clothing (as it's cold at the summit), a flashlight, water, and high-energy food. The user decided to leave at 4:00 AM on Saturday, March 16, 2024 to catch the sunrise and planned to invite friends for the adventure. They were very excited about the trip, hoping to connect with nature.",
    "timestamp": "2024-03-14T15:00:00"
}} 

    \end{tcolorbox}
        \label{tab:prompt for Episodic Generation}
\end{table*}

\begin{table*}[!htbp]

    \centering
    \small
    \begin{tcolorbox}[width=\textwidth, title={\textbf{Prompt for Memory Component Generation}}]

    You are an AI memory system. Extract HIGH-VALUE, PERSISTENT semantic memories as memory components from the following segments.

CRITICAL: Focus on extracting LONG-TERM VALUABLE KNOWLEDGE, not temporary conversation details.

Segments to analyze:
{segments}

\#\# HIGH-VALUE Knowledge Criteria

Extract ONLY knowledge that passes these tests:
- **Persistence Test**: Will this still be true in 6 months?
- **Specificity Test**: Does it contain concrete, searchable information?
- **Utility Test**: Can this help predict future user needs?
- **Independence Test**: Can be understood without conversation context?

\#\# HIGH-VALUE Categories (FOCUS ON THESE):

1. **Identity \& Professional**
   - Names, titles, companies, roles
   - Education, qualifications, skills
   
2. **Persistent Preferences**  
   - Favorite books, movies, music, tools
   - Technology preferences with reasons
   - Long-term likes and dislikes
   
3. **Technical Knowledge**
   - Technologies used (with versions)
   - Architectures, methodologies
   - Technical decisions and rationales
   
4. **Relationships**
   - Names of family, colleagues, friends
   - Team structure, reporting lines
   - Professional networks
   
5. **Goals \& Plans**
   - Career objectives
   - Learning goals
   - Project plans
   
6. **Patterns \& Habits**
   - Regular activities
   - Workflows, schedules
   - Recurring challenges

\#\# Examples:

HIGH-VALUE (Extract these):
- "Caroline's favorite book is 'Becoming Nicole' by Amy Ellis Nutt"
- "The user works at ByteDance as a senior ML engineer"
- "The user prefers PyTorch over TensorFlow for debugging"
- "The user's team lead is named Sarah"
- "The user is learning Rust for systems programming"
- "The user has been practicing yoga since March 2021"
- "The user joined Amazon in August 2020 as a data scientist"
- "The user plans to relocate to Seattle in January 2025"

LOW-VALUE (Skip these):
- "The user thanked the assistant"
- "The user was confused about X"
- "The user appreciated the help"
- "The conversation was productive"
- Any temporary emotions or reactions

Quality over quantity - extract only knowledge that truly helps understand the user long-term.

    \end{tcolorbox}
        \label{tab:prompt for Semantic Generation}
\end{table*}

\begin{table*}[!htbp]

    \centering
    \small
    \begin{tcolorbox}[width=\textwidth, title={\textbf{Prompt for group Generation}}]

    "You are building a stable high-level group summary that captures recurring facts."
    "Given the following factual statements, write a concise abstract group that names the topic:"
    f"{joined}"
    "Return only the summary."

    \end{tcolorbox}
        \label{tab:prompt for Theme Generation}
\end{table*}

\subsection{Answer Generation}
We provide dataset-specific answer generation prompts to match the distinct answer formats and evaluation protocols of LoCoMo and PerLTQA.
For LoCoMo, the prompt enforces short phrase-style outputs and explicit handling of temporal reasoning by resolving relative time expressions using memory timestamps.
For PerLTQA, the prompt requests a single complete sentence and prioritises explicit factual attributes from the memory store, resolving conflicts by recency.
Importantly, to ensure fair comparison, we keep the final answer prompt identical across different baselines within each dataset.

\begin{table*}[!htbp]

    \centering
    \small
    \begin{tcolorbox}[width=\textwidth, title={\textbf{Prompt for Answer Generation in LOCOMO Dataset}}]

    You are an intelligent memory assistant tasked with retrieving accurate information from conversation memories.

    \# CONTEXT:
    You have access to memories from two speakers in a conversation. These memories contain
    timestamped information that may be relevant to answering the question.

    \# INSTRUCTIONS:
    1. Carefully analyze all provided memories from both speakers
    2. Pay special attention to the timestamps to determine the answer
    3. If the question asks about a specific event or fact, look for direct evidence in the memories
    4. If the memories contain contradictory information, prioritize the most recent memory
    5. If there is a question about time references (like "last year", "two months ago", etc.),
       calculate the actual date based on the memory timestamp. For example, if a memory from
       4 May 2022 mentions "went to India last year," then the trip occurred in 2021.
    6. Always convert relative time references to specific dates, months, or years. For example,
       convert "last year" to "2022" or "two months ago" to "March 2023" based on the memory
       timestamp. Ignore the reference while answering the question. 
       Do not answer temporal questions in timestamp such as "2023-05-08", but answer naturally such as "7 May 2023" or "The week before 6 July 2023"!   
    7. Focus only on the content of the memories from both speakers. Do not confuse character
       names mentioned in memories with the actual users who created those memories.
    8. The answer should be less than 5-8 words.
    
    Component Memories:
    {{ components }}

    Segment Memories:
    {{ segments }}

    Question: {{ question }}

    Answer:

    \end{tcolorbox}
        \label{tab:prompt for Answer Generation in LOCOMO Dataset}
\end{table*}

\begin{table*}[!htbp]

    \centering
    \small
    \begin{tcolorbox}[width=\textwidth, title={\textbf{Prompt for Answer Generation in PerLTQA Dataset}}]

You are a memory question-answering assistant. Answer the question using ONLY the information in the provided memories.

Memories may contain profiles, events, and relationship descriptions. Prefer explicit facts (e.g., "Gender: female", "Nationality: China", dates, numbers, names).
If multiple memories conflict, choose the most recent one by timestamp.

OUTPUT RULES (VERY IMPORTANT):
- Write ONE concise, complete sentence.
- Preserve the key value EXACTLY as stated in the memories (names, numbers, dates, gender terms, nationality words).
- Do NOT add extra explanation.
- Do NOT include quotes, bullet points, or prefixes like "Answer:".

Segment Memories:
{{ Segment }}

Component Memories:
{{ component }}

Question: {{ question }}

Answer:

    \end{tcolorbox}
        \label{tab:prompt for Answer Generation in PerLTQA Dataset}
\end{table*}


\newpage

\end{document}